\documentclass{article}

\usepackage{times}
\usepackage{graphicx} 
\usepackage{subfigure}

\usepackage{natbib}

\usepackage{algorithm}
\usepackage{algorithmic}
\usepackage[algo2e,linesnumbered,lined,boxed,vlined]{algorithm2e}

\usepackage{amsfonts}
\usepackage{amssymb}

\usepackage{hyperref}

\usepackage[accepted]{icml2016}

\usepackage{comment}
\usepackage{amsthm}
\usepackage{mathtools}
\usepackage{tabularx}
\usepackage{multirow}

\newtheorem{theorem}{Theorem}
\newtheorem{lemma}{Lemma}

\newtheorem{corollary}{Corollary}
\newtheorem{definition}{Definition}

\def\b{\ensuremath\boldsymbol}

\usepackage{lastpage}
\usepackage{fancyhdr}
\usepackage{url}
\icmltitlerunning{}

\begin{document}

\AddToShipoutPictureBG*{%
  \AtPageUpperLeft{%
    \setlength\unitlength{1in}%
    \hspace*{\dimexpr0.5\paperwidth\relax}
    \makebox(0,-0.75)[c]{\normalsize {\color{black} To appear as a part of an upcoming textbook on dimensionality reduction and manifold learning.}}
    }}

\twocolumn[
\icmltitle{Johnson-Lindenstrauss Lemma, Linear and Nonlinear Random Projections, Random Fourier Features, and Random Kitchen Sinks: Tutorial and Survey}

\icmlauthor{Benyamin Ghojogh}{bghojogh@uwaterloo.ca}
\icmladdress{Department of Electrical and Computer Engineering, 
\\Machine Learning Laboratory, University of Waterloo, Waterloo, ON, Canada}
\icmlauthor{Ali Ghodsi}{ali.ghodsi@uwaterloo.ca}
\icmladdress{Department of Statistics and Actuarial Science \& David R. Cheriton School of Computer Science, 
\\Data Analytics Laboratory, University of Waterloo, Waterloo, ON, Canada}
\icmlauthor{Fakhri Karray}{karray@uwaterloo.ca}
\icmladdress{Department of Electrical and Computer Engineering, 
\\Centre for Pattern Analysis and Machine Intelligence, University of Waterloo, Waterloo, ON, Canada}
\icmlauthor{Mark Crowley}{mcrowley@uwaterloo.ca}
\icmladdress{Department of Electrical and Computer Engineering, 
\\Machine Learning Laboratory, University of Waterloo, Waterloo, ON, Canada}

\icmlkeywords{Tutorial}

\vskip 0.3in
]

\begin{abstract}
This is a tutorial and survey paper on the Johnson-Lindenstrauss (JL) lemma and linear and nonlinear random projections. We start with linear random projection and then justify its correctness by JL lemma and its proof. Then, sparse random projections with $\ell_1$ norm and interpolation norm are introduced. Two main applications of random projection, which are low-rank matrix approximation and approximate nearest neighbor search by random projection onto hypercube, are explained. Random Fourier Features (RFF) and Random Kitchen Sinks (RKS) are explained as methods for nonlinear random projection. Some other methods for nonlinear random projection, including extreme learning machine, randomly weighted neural networks, and ensemble of random projections, are also introduced. 
\end{abstract}

\section{Introduction}

Linear dimensionality reduction methods project data onto the low-dimensional column space of a projection matrix.
Many of these methods, such as Principal Component Analysis (PCA) \cite{ghojogh2019unsupervised} and Fisher Discriminant Analysis (FDA) \cite{ghojogh2019fisher}, learn a projection matrix for either better representation of data or discrimination of classes in the subspace. 
For example, PCA learns the projection matrix to maximize the variance of projected data. FDA learns the projections matrix to discriminate classes in the subspace. 
However, it was seen that if we do not learn the projection matrix and just sample the elements of projection matrix randomly, it still works. 
In fact, random projection preserves the distances of points with a small error after projection. 
Hence, this random projection works very well surprisingly although its projection matrix is not learned! 
For justifying why random projection works well, the Johnson-Lindenstrauss (JL) lemma \cite{johnson1984extensions} was proposed which bounds the error of random projection. 
Random projection is a probabilistic dimensionality reduction method. 
Theories for various random projection methods have a similar approach (e.g., see the proofs in this paper) and similar bounds. 
In terms of dealing with probability bounds, it can be slightly related to the Probably Approximately Correct (PAC) learning \cite{shalev2014understanding}. 
A survey on linear random projection is \cite{xie2017survey} and a good book on this topic is \cite{vempala2005random}. Slides \cite{liberty2006random,ward2014dimension} are also good slides about linear random projection. 

The theory of linear random projection has been developed during years; however, the theory of nonlinear random projection still needs to be improved because it is more complicated to analyze than linear random projection. 
Nonlinear random projection can be modeled as a linear random projection followed by a nonlinear function. 
Random Fourier Features (RFF) \cite{rahimi2007random} and Random Kitchen Sinks (RKS) \cite{rahimi2008weighted,rahimi2008uniform} are two initial works on nonlinear random projection. Some other methods have also been proposed for nonlinear random projection. They are Extreme Learning Machine (ELM) \cite{huang2006extreme,huang2011extreme,tang2015extreme}, randomly weighted neural network \cite{jarrett2009best}, and ensemble random projections \cite{schclar2009random,cannings2015random,karimi2017ensembles}. 

This is a tutorial and survey paper on linear and nonlinear random projection.
The remainder of this paper is organized as follows. We introduce linear projection and linear random projection in Section \ref{section_linear_random_projection}. The JL lemma and its proof are provided in Section \ref{section_JL_lemma}. Then, sparse random projection by $\ell_1$ norm is introduced in Section \ref{section_sparse_linear_random_projection}. Applications of linear random projection, including low-rank matrix approximation and approximate nearest neighbor search, are explained in Section \ref{section_applications_of_random_projection}. Section \ref{section_RFF_and_RKS} introduce the theory of RFF and RKS for nonlinear random projection. Some other methods for nonlinear random projection, including ELM, random neural network, and ensemble of random projections, are introduced in Section \ref{section_other_nonlinear_random_projections}. Finally, Section \ref{section_conclusion} concludes the paper. 

\section*{Required Background for the Reader}

This paper assumes that the reader has general knowledge of calculus, probability, linear algebra, and basics of optimization. 

\section{Linear Random Projection}\label{section_linear_random_projection}

\subsection{Linear Projection}

\subsubsection{Projection Formulation}

Assume we have a data point $\b{x} \in \mathbb{R}^d$. We want to project this data point onto the vector space spanned by $p$ vectors $\{\b{u}_1, \dots, \b{u}_p\}$ where each vector is $d$-dimensional and usually $p \ll d$. We stack these vectors column-wise in matrix $\b{U} = [\b{u}_1, \dots, \b{u}_p] \in \mathbb{R}^{d \times p}$. In other words, we want to project $\b{x}$ onto the column space of $\b{U}$, denoted by $\mathbb{C}\text{ol}(\b{U})$.

The projection of $\b{x} \in \mathbb{R}^d$ onto $\mathbb{C}\text{ol}(\b{U}) \in \mathbb{R}^p$ and then its representation in the $\mathbb{R}^d$ (its reconstruction) can be seen as a linear system of equations:
\begin{align}\label{equation_projection}
\mathbb{R}^d \ni \widehat{\b{x}} := \b{U \beta},
\end{align}
where we should find the unknown coefficients $\b{\beta} \in \mathbb{R}^p$. 

If the $\b{x}$ lies in the $\mathbb{C}\text{ol}(\b{U})$ or $\textbf{span}\{\b{u}_1, \dots, \b{u}_p\}$, this linear system has exact solution, so $\widehat{\b{x}} = \b{x} = \b{U \beta}$. However, if $\b{x}$ does not lie in this space, there is no any solution $\b{\beta}$ for $\b{x} = \b{U \beta}$ and we should solve for projection of $\b{x}$ onto $\mathbb{C}\text{ol}(\b{U})$ or $\textbf{span}\{\b{u}_1, \dots, \b{u}_p\}$ and then its reconstruction. In other words, we should solve for Eq. (\ref{equation_projection}). In this case, $\widehat{\b{x}}$ and $\b{x}$ are different and we have a residual:
\begin{align}\label{equation_residual_1}
\b{r} = \b{x} - \widehat{\b{x}} = \b{x} - \b{U \beta},
\end{align}
which we want to be small. As can be seen in Fig. \ref{figure_residual_and_space}, the smallest residual vector is orthogonal to $\mathbb{C}\text{ol}(\b{U})$; therefore:
\begin{align}
\b{x} - \b{U\beta} \perp \b{U} &\implies \b{U}^\top (\b{x} - \b{U \beta}) = 0, \nonumber \\
& \implies \b{\beta} = (\b{U}^\top \b{U})^{-1} \b{U}^\top \b{x}. \label{equation_beta}
\end{align}
It is noteworthy that the Eq. (\ref{equation_beta}) is also the formula of coefficients in linear regression where the input data are the rows of $\b{U}$ and the labels are $\b{x}$; however, our goal here is different. 

\begin{figure}[!t]
\centering
\includegraphics[width=2.2in]{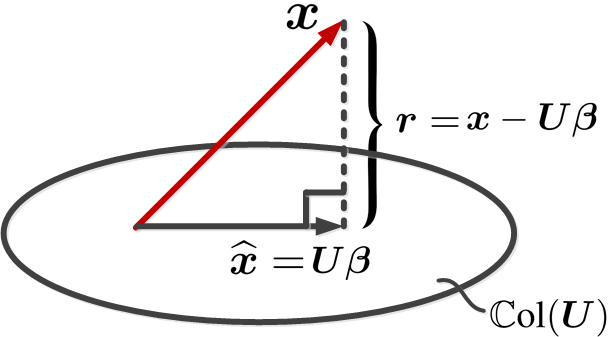}
\caption{The residual and projection onto the column space of $\b{U}$. Credit of image is for \cite{ghojogh2019unsupervised}.}
\label{figure_residual_and_space}
\end{figure}

Plugging Eq. (\ref{equation_beta}) in Eq. (\ref{equation_projection}) gives us:
\begin{align*}
\widehat{\b{x}} = \b{U} (\b{U}^\top \b{U})^{-1} \b{U}^\top \b{x}.
\end{align*}
We define:
\begin{align}\label{equation_hat_matrix}
\mathbb{R}^{d \times d} \ni \b{\Pi} := \b{U} (\b{U}^\top \b{U})^{-1} \b{U}^\top,
\end{align}
as ``projection matrix'' because it projects $\b{x}$ onto $\mathbb{C}\text{ol}(\b{U})$ (and reconstructs back).
Note that $\b{\Pi}$ is also referred to as the ``hat matrix'' in the literature because it puts a hat on top of $\b{x}$.

If the vectors $\{\b{u}_1, \dots, \b{u}_p\}$ are orthonormal (the matrix $\b{U}$ is orthogonal), we have $\b{U}^\top = \b{U}^{-1}$ and thus $\b{U}^\top \b{U} = \b{I}$. Therefore, Eq. (\ref{equation_hat_matrix}) is simplified:
\begin{align}
& \b{\Pi} = \b{U} \b{U}^\top.
\end{align}
So, we have:
\begin{align}\label{equation_x_hat}
\widehat{\b{x}} = \b{\Pi}\, \b{x} = \b{U} \b{U}^\top \b{x}.
\end{align}

\subsubsection{Projection onto a Subspace}\label{section_projection_subspace}

In subspace learning, the projection of a vector $\b{x} \in \mathbb{R}^d$ onto the column space of $\b{U} \in \mathbb{R}^{d \times p}$ (a $p$-dimensional subspace spanned by $\{\b{u}_j\}_{j=1}^p$ where $\b{u}_j \in \mathbb{R}^d$) is defined as:
\begin{align}
&\mathbb{R}^{p} \ni \widetilde{\b{x}} := \b{U}^\top \b{x}, \label{equation_projection_training_onePoint_severalDirections} \\
&\mathbb{R}^{d} \ni \widehat{\b{x}} := \b{U}\b{U}^\top \b{x} = \b{U} \widetilde{\b{x}}, \label{equation_reconstruction_training_onePoint_severalDirections}
\end{align}
where $\widetilde{\b{x}}$ and $\widehat{\b{x}}$ denote the projection and reconstruction of $\b{x}$, respectively.

If we have $n$ data points, $\{\b{x}_i\}_{i=1}^n$, which can be stored column-wise in a matrix $\b{X} \in \mathbb{R}^{d \times n}$, the projection and reconstruction of $\b{X}$ are defined as:
\begin{align}
&\mathbb{R}^{p \times n} \ni \widetilde{\b{X}} := \b{U}^\top \b{X}, \label{equation_projection_training_SeveralPoints_severalDirections} \\
&\mathbb{R}^{d \times n} \ni \widehat{\b{X}} := \b{U}\b{U}^\top \b{X} = \b{U} \widetilde{\b{X}}, \label{equation_reconstruction_training_SeveralPoints_severalDirections}
\end{align}
respectively.

If we have an out-of-sample data point $\b{x}_t$ which was not used in calculation of $\b{U}$, the projection and reconstruction of it are defined as:
\begin{align}
&\mathbb{R}^{p} \ni \widetilde{\b{x}}_t := \b{U}^\top \b{x}_t, \label{equation_projection_outOfSample_onePoint_severalDirections} \\
&\mathbb{R}^{d} \ni \widehat{\b{x}}_t := \b{U}\b{U}^\top \b{x}_t = \b{U} \widetilde{\b{x}}_t, \label{equation_reconstruction_outOfSample_onePoint_severalDirections}
\end{align}
respectively.

In case we have $n_t$ out-of-sample data points, $\{\b{x}_{t,i}\}_{i=1}^{n_t}$, which can be stored column-wise in a matrix $\b{X}_t \in \mathbb{R}^{d \times n_t}$, the projection and reconstruction of $\b{X}_t$ are defined as:
\begin{align}
&\mathbb{R}^{p \times n_t} \ni \widetilde{\b{X}}_t := \b{U}^\top \b{X}_t, \label{equation_projection_outOfSample_SeveralPoints_severalDirections} \\
&\mathbb{R}^{d \times n_t} \ni \widehat{\b{X}}_t := \b{U}\b{U}^\top \b{X}_t = \b{U} \widetilde{\b{X}}_t, \label{equation_reconstruction_outOfSample_SeveralPoints_severalDirections}
\end{align}
respectively.

For the properties of the projection matrix $\b{U}$, refer to \cite{ghojogh2019unsupervised}.

\subsection{Linear Random Projection}

Linear random projection is projection of data points onto the column space of a projection matrix where the elements of projection matrix are i.i.d. random variables sampled from a distribution with zero mean and (possibly scaled) unit variance. In other words, random projection is a function $f\!: \mathbb{R}^d \rightarrow \mathbb{R}^p$, $f\!: \b{x} \mapsto \b{U}^\top \b{x}$:
\begin{align}\label{equation_linear_random_projection}
\mathbb{R}^p \ni f(\b{x}) := \b{U}^\top \b{x} = \sum_{t=1}^p \b{u}_t^\top \b{x} = \sum_{j=1}^d \sum_{t=1}^p u_{jt}\, x_j,
\end{align}
where $\b{U} = [\b{u}_1, \dots, \b{u}_p] \in \mathbb{R}^{d \times p}$ is the random projection matrix, $u_{jt}$ is the $(j,t)$-th element of $\b{U}$, and $x_j$ is the $j$-th element of $\b{x} \in \mathbb{R}^d$. 
The elements of $\b{U}$ are sampled from a distribution with zero mean and (possibly scaled) unit variance. For example, we can use the Gaussian distribution $u_{j,t} \sim \mathcal{N}(0, 1), \forall j,t$. 
Some example distributions to use for random projection are Gaussian \cite{giryes2016deep} and Cauchy \cite{li2007nonlinear,ramirez2012reconstruction} distributions. 

It is noteworthy that, in some papers, the projection is normalized:
\begin{align}\label{equation_linear_random_projection_normalized}
f\!: \b{x} \mapsto \frac{1}{\sqrt{p}} \b{U}^\top \b{x}.
\end{align}
If we have a dataset of $d$-dimensional points with sample size $n$, we can stack the points in $\b{X} := [\b{x}_1, \dots, \b{x}_n] \in \mathbb{R}^{d \times n}$. 
Eq. (\ref{equation_linear_random_projection}) or (\ref{equation_linear_random_projection_normalized}) is stated as:
\begin{align}
\mathbb{R}^{p \times n} \ni f(\b{X}) := \b{U}^\top \b{X} \quad \text{or} \quad \frac{1}{\sqrt{p}} \b{U}^\top \b{X}.
\end{align}

In contrast to other linear dimensionality reduction methods which learn the projection matrix using training dataset for better data representation of class discrimination, random projection does not learn the projection matrix but randomly samples it independent of data.  
Surprisingly, projection of data onto the column space of this random matrix works very well although the projection matrix is completely random and independent of data. 
The JL lemma \cite{johnson1984extensions} justifies why random projection works. In the following, we introduce this lemma and its proof. 

\section{The Johnson-Lindenstrauss Lemma}\label{section_JL_lemma}


\begin{theorem}[Johnson-Lindenstrauss Lemma \cite{johnson1984extensions}]\label{theorem_JL_lemma}
For any set $\mathcal{X} := \{\b{x}_i \in \mathbb{R}^d\}_{i=1}^n$, any integer $n$ as the sample size, and any $0<\epsilon<1$ as error tolerance, let $p$ be a positive integer satisfying:
\begin{align}\label{equation_JL_lemma_p_lower_bound}
p \geq \Omega(\frac{\ln(n)}{\epsilon^2 - \epsilon^3}),
\end{align}
where $\ln(.)$ is the natural logarithm and $\Omega(.)$ is the lower bound complexity [n.b. some works state Eq. (\ref{equation_JL_lemma_p_lower_bound}) as:
\begin{align}\label{equation_JL_lemma_p_lower_bound_relaxed}
p \geq \Omega(\epsilon^{-2} \ln(n)).
\end{align}
by ignoring $\epsilon^3$ against $\epsilon^2$ in the denominator because $\epsilon \in (0,1)$]. 
There exists a linear map $f\!: \mathbb{R}^d \rightarrow \mathbb{R}^p$, $f\!: \b{x} \mapsto \b{U}^\top \b{x}$, with projection matrix $\b{U} = [\b{u}_1, \dots, \b{u}_p] \in \mathbb{R}^{d \times p}$, such that we have:
\begin{equation}\label{equation_JL_lemma_correct_projection}
\begin{aligned}
(1 - \epsilon) \|\b{x}_i - \b{x}_j\|_2^2 \leq \|&f(\b{x}_i) - f(\b{x}_j)\|_2^2 \\
&\leq (1 + \epsilon) \|\b{x}_i - \b{x}_j\|_2^2,
\end{aligned}
\end{equation}
for all $\b{x}_i, \b{x}_j \in \mathcal{X}$, with probability of success as:
\begin{align}\label{equation_random_projection_prob_success}
\mathbb{P}\Big(
(1 - \epsilon) \|\b{x}_i - &\b{x}_j\|_2^2 \leq \|f(\b{x}_i) - f(\b{x}_j)\|_2^2 \nonumber \\
&\leq (1 + \epsilon) \|\b{x}_i - \b{x}_j\|_2^2 
\Big) \geq 1 - \delta,
\end{align}
where $\delta := 2 e^{-(\epsilon^2 - \epsilon^3) (p/4)}$ and the elements of the projection matrix are i.i.d. random variables with mean zero and (scaled) unit variance. An example is $u_{ij} \sim \mathcal{N}(0, 1/p) = (1/\sqrt{p}) \mathcal{N}(0, 1)$ where $u_{ij}$ denotes the $(i,j)$-th element of $\b{U}$. 
\end{theorem}
\begin{proof}
The proof is based on \cite{karimi2018exploring,indyk1998approximate,dasgupta1999elementary,dasgupta2003elementary,achlioptas2003database,achlioptas2001database,matouvsek2008variants}.
Also some other proofs exists for JL lemma, such as \cite{frankl1988johnson,arriaga1999algorithmic}. 

\hfill \break
The proof can be divided into three steps:

\textbf{Step 1 of proof:} 
In this step, we show that the random projection preserves local distances of points in expectation.
Let $\b{x}_i = [x_{i1}, \dots, x_{id}]^\top \in \mathbb{R}^d$. We have:
\begin{align}
&\mathbb{E}\big[\|f(\b{x}_i) - f(\b{x}_j)\|_2^2\big] = \mathbb{E}\big[\|\b{U}^\top \b{x}_i - \b{U}^\top \b{x}_j\|_2^2\big] \nonumber\\
&\overset{(a)}{=} \mathbb{E}\Bigg[\sum_{t=1}^p \big(\b{u}_t^\top \b{x}_i - \b{u}_t^\top \b{x}_j\big)^2\Bigg] \nonumber\\
&\overset{(b)}{=} \sum_{t=1}^p \mathbb{E}\big[\big(\b{u}_t^\top \b{x}_i - \b{u}_t^\top \b{x}_j\big)^2\big] \nonumber\\
&= \sum_{t=1}^p \Big( \mathbb{E}\big[\big(\b{u}_t^\top \b{x}_i)^2\big] + \mathbb{E}\big[(\b{u}_t^\top \b{x}_j\big)^2\big] \nonumber\\
&~~~~~~~~~~~ - 2\,\mathbb{E}\big[(\b{u}_t^\top \b{x}_i \b{u}_t^\top \b{x}_j\big)^2\big] \Big), \label{equation_expectation_difference_of_mapping}
\end{align}
where $(a)$ is because of definition of $\ell_2$ norm and $(b)$ is because expectation is a linear operator. 
We know $u_{ij}$ has zero mean and unit variance. Let its unit variance be normalized as $1/p$. Then, the covariance of elements of projection matrix is:
\begin{align}\label{equation_covariance_of_projection_matrix}
\mathbb{E}[u_{kt} u_{lt}] = \delta_{kl} := 
\left\{
    \begin{array}{ll}
        0 & \mbox{if } i \neq j \\
        1/p & \mbox{if } i = j,
    \end{array}
\right.
\end{align}
which is a normalized Kronekcer delta.
The first term in this expression is:
\begin{align*}
&\mathbb{E}\big[(\b{u}_t^\top \b{x}_i)^2\big] = \mathbb{E}\big[\big(\sum_{k=1}^d u_{kt} x_{ik}\big)^2\big] \\
&= \mathbb{E}\Big[\sum_{k=1}^d \sum_{l=1}^d u_{kt} u_{lt} x_{ik} x_{il}\Big] \overset{(a)}{=} \sum_{k=1}^d \sum_{l=1}^d \mathbb{E}[u_{kt} u_{lt} x_{ik} x_{il}] \\
&\overset{(b)}{=} \sum_{k=1}^d \sum_{l=1}^d x_{ik} x_{il} \mathbb{E}[u_{kt} u_{lt}] \overset{(\ref{equation_covariance_of_projection_matrix})}{=} \sum_{k=1}^d \sum_{l=1}^d x_{ik} x_{il} \delta_{kl} \\
&= \frac{1}{p} \sum_{k=1}^d x_{ik}^2 = \frac{\|\b{x}_i\|_2^2}{p},
\end{align*}
where $(a)$ is because of linearity of expectation and $(b)$ is because data are deterministic.
Likewise, the second term of Eq. (\ref{equation_expectation_difference_of_mapping}) is $\mathbb{E}\big[(\b{u}_t^\top \b{x}_j)^2\big] = \|\b{x}_j\|_2^2 / p$. 
The third term is:
\begin{align*}
&\mathbb{E}\big[(\b{u}_t^\top \b{x}_i \b{u}_t^\top \b{x}_j\big)^2\big] = \mathbb{E}\big[\big(\sum_{k=1}^d u_{kt} x_{ik}\big) \big(\sum_{l=1}^d u_{lt} x_{jl}\big)\big] \\
&= \mathbb{E}\Big[\sum_{k=1}^d \sum_{l=1}^d u_{kt} u_{lt} x_{ik} x_{jl}\Big] = \sum_{k=1}^d \sum_{l=1}^d \mathbb{E}[u_{kt} u_{lt} x_{ik} x_{jl}] \\
&= \sum_{k=1}^d \sum_{l=1}^d x_{ik} x_{jl} \mathbb{E}[u_{kt} u_{lt}] \overset{(\ref{equation_covariance_of_projection_matrix})}{=} \sum_{k=1}^d \sum_{l=1}^d x_{ik} x_{jl} \delta_{kl} \\
&= \frac{1}{p} \sum_{k=1}^d x_{ik} x_{jk} = \frac{\b{x}_i^\top \b{x}_j}{p}.
\end{align*}
Therefore, Eq. (\ref{equation_expectation_difference_of_mapping}) is simplified to:
\begin{align}
\mathbb{E}\big[\|f(\b{x}_i) &- f(\b{x}_j)\|_2^2\big] \nonumber\\
&= \sum_{t=1}^p \Big( \frac{\|\b{x}_i\|_2^2}{p} + \frac{\|\b{x}_j\|_2^2}{p} - 2\frac{\b{x}_i^\top \b{x}_j}{p} \Big) \nonumber\\
&= \sum_{t=1}^p \Big( \frac{\|\b{x}_i - \b{x}_j\|_2^2}{p} \Big) = \|\b{x}_i - \b{x}_j\|_2^2 \sum_{t=1}^p \frac{1}{p} \nonumber\\
&= \|\b{x}_i - \b{x}_j\|_2^2, \quad \forall i, j \in \{1, \dots, n\}.
\end{align}
This shows that this projection by random projection matrix preserves the local distances of data points in expectation. 

\hfill \break
\textbf{Step 2 of proof:}
In this step, we show that the variance of local distance preservation is bounded and small for one pair of points. The probability of error in local distance preservation is:
\begin{align*}
&\mathbb{P}\Big(\|f(\b{x}_i) - f(\b{x}_j)\|_2^2 > (1+\epsilon) \|\b{x}_i - \b{x}_j\|_2^2\Big) \\
&= \mathbb{P}\Big(\|\b{U}^\top \b{x}_i - \b{U}^\top \b{x}_j\|_2^2 > (1+\epsilon) \|\b{x}_i - \b{x}_j\|_2^2\Big) \\
&= \mathbb{P}\Big(\|\b{U}^\top (\b{x}_i - \b{x}_j)\|_2^2 > (1+\epsilon) \|\b{x}_i - \b{x}_j\|_2^2\Big).
\end{align*}
Let $\b{x}_D := \b{x}_i - \b{x}_j$ so:
\begin{align*}
\mathbb{P}\Big(\|\b{U}^\top \b{x}_D\|_2^2 > (1+\epsilon) \|\b{x}_d\|_2^2\Big).
\end{align*}
We know that $u_{ij} \sim (1/\sqrt{p}) \mathcal{N}(0, 1)$. 
So, the projected data are $\mathbb{R}^p \ni \b{y}_D = \b{U}^\top \b{x}_D$ and we have $\b{y}_D \sim (1/\sqrt{p}) \mathcal{N}(0, \|\b{x}_D\|_2^2) = \mathcal{N}(0, (1/p)\|\b{x}_D\|_2^2)$ because of the quadratic characteristic of variance. 

We define a matrix $\b{A} = [a_{ij}] \in \mathbb{R}^{d \times p}$ where $a_{ij} \sim \mathcal{N}(0, 1)$. Assume $\b{z}_D = [z_1, \dots, z_p]^\top := (\b{A}^\top \b{x}_D) / \|\b{x}_D\|_2$. Hence:
\begin{align*}
&\mathbb{P}\Big(\|\frac{1}{\sqrt{p}}\b{A}^\top \b{x}_D\|_2^2 > (1+\epsilon) \|\b{x}_D\|_2^2\Big) \\
&= \mathbb{P}\Big(\frac{1}{p}\|\b{A}^\top \b{x}_D\|_2^2 > (1+\epsilon) \|\b{x}_D\|_2^2\Big) \\
&\overset{(a)}{=} \mathbb{P}\Big(\frac{1}{\|\b{x}_D\|_2^2} \|\b{A}^\top \b{x}_D\|_2^2 > (1+\epsilon) p\Big) \\
&= \mathbb{P}\Big(\big\|\frac{\b{A}^\top \b{x}_D}{\|\b{x}_D\|_2}\big\|_2^2 > (1+\epsilon) p\Big) \\
&\overset{(b)}{=} \mathbb{P}\Big(\|\b{z}_D\|_2^2 > (1+\epsilon) p\Big) \\
&= \mathbb{P}\Big(\sum_{k=1}^p z_k^2 > (1+\epsilon)\, p\Big) \overset{(c)}{=} \mathbb{P}\Big(\chi_p^2 > (1+\epsilon)\, p\Big),
\end{align*}
where $(a)$ is because the sides are multiplied by $p / \|\b{x}_D\|_2^2$, $(b)$ is because of the definition of $\b{z}_D$, and $(c)$ is because the summation of $p$ squared standard normal distributions is the chi-squared distribution with $p$ degrees of freedom, denoted by $\chi_p^2$. 
We have:
\begin{align}
\mathbb{P}\Big(\chi_p^2 > &(1+\epsilon)\, p\Big) \overset{(a)}{=} \mathbb{P}\Big(e^{\lambda \chi_p^2} > e^{\lambda (1+\epsilon)\, p}\Big) \nonumber\\
&\overset{(b)}{\leq} \frac{\mathbb{E}[e^{\lambda \chi_p^2}]}{e^{\lambda (1+\epsilon)\, p}} \overset{(c)}{=} \frac{(1-2\lambda)^{-p/2}}{e^{\lambda (1+\epsilon)\, p}} \label{equation_prob_error_1}
\end{align}
where $(a)$ is because of taking power of $e$ with parameter $\lambda \geq 0$, $(b)$ is because of the Markov inequality (i.e., $\mathbb{P}(x \geq a) \leq \mathbb{E}[x]/a$), and $(c)$ is because of the moment generating function of the chi-squared distribution $x \sim \chi_p^2$ which is $M_x(t) = \mathbb{E}[e^{\lambda x}] = (1-2\lambda)^{-p/2}$. 
We want to minimize the probability of error. So, we minimize this error in Eq. (\ref{equation_prob_error_1}) with respect to the parameter $\lambda$:
\begin{align*}
\frac{\partial}{\partial \lambda} (\frac{(1-2\lambda)^{-p/2}}{e^{\lambda (1+\epsilon)\, p}}) \overset{\text{set}}{=} 0 \implies \lambda = \frac{\epsilon}{2(1+\epsilon)}.
\end{align*}
Substituting this equation in Eq. (\ref{equation_prob_error_1}) gives:
\begin{align*}
\mathbb{P}&\Big(\chi_p^2 > (1+\epsilon)\, p\Big) = \frac{(1+\epsilon)^{p/2}}{e^{(p \epsilon) / 2}} \\
&= ((1+\epsilon) e^{-\epsilon})^{p/2} \overset{(a)}{\leq} e^{-(p/4) (\epsilon^2 - \epsilon^3)},
\end{align*}
where $(a)$ is because $1+\epsilon \leq e^{(\epsilon - (\epsilon^2 - \epsilon^3)/2)}$. 
Above, we derived the probability of $\mathbb{P}(\|f(\b{x}_i) - f(\b{x}_j)\|_2^2 > (1+\epsilon) \|\b{x}_i - \b{x}_j\|_2^2)$.
Likewise, we can derive the same amount of probability for $\mathbb{P}(\|f(\b{x}_i) - f(\b{x}_j)\|_2^2 < (1-\epsilon) \|\b{x}_i - \b{x}_j\|_2^2)$. 
For the projection of pair $\b{x}_i$ and $\b{x}_j$, we want to have:
\begin{align}
\|f(\b{x}_i) - f(\b{x}_j)\|_2^2 \approx \|\b{x}_i - \b{x}_j\|_2^2.
\end{align}
Hence, the total probability of error for the projection of pair $\b{x}_i$ and $\b{x}_j$ is:
\begin{align}
\mathbb{P}\Big(&\Big[\|f(\b{x}_i) - f(\b{x}_j)\|_2^2 > (1+\epsilon) \|\b{x}_i - \b{x}_j\|_2^2\Big] \nonumber\\
&\cup
\Big[\|f(\b{x}_i) - f(\b{x}_j)\|_2^2 < (1-\epsilon) \|\b{x}_i - \b{x}_j\|_2^2\Big] \Big) \nonumber\\
&\leq 2\, e^{-(p/4) (\epsilon^2 - \epsilon^3)}. \label{equation_JL_prob_error_twoPoints}
\end{align}
This shows that the probability of error for every pair $\b{x}_i$ and $\b{x}_j$ is bounded. This also proves the Eq. (\ref{equation_random_projection_prob_success}).

\hfill \break
\textbf{Step 3 of proof:}
In this step, we bound the probability of error for all pairs of data points:
We use the Bonferroni's inequality or the so-called union bound \cite{bonferroni1936teoria}:
\begin{align}\label{equation_Bonferroni_inequality}
\mathbb{P}\Big(\bigcup_i E_i\Big) \leq \sum_{i} \mathbb{P}(E_i),
\end{align}
where $E_i$ is the $i$-th probability event. We have $\binom{n}{2} = n(n-1)/2$ pairs of points. Hence, the total probability of error for all pairs of points is:
\begin{align}
\mathbb{P}(\text{error}) &\leq \frac{n(n-1)}{2} \times 2\, e^{-(p/4) (\epsilon^2 - \epsilon^3)} \nonumber\\
&= n(n-1) e^{-(p/4) (\epsilon^2 - \epsilon^3)} \leq \delta', \label{equation_prob_error_JL_lemma}
\end{align}
where the upper-bound $0 \leq \delta' \ll 1$ can be selected according to $n$, $p$, and $\epsilon$. Therefore, the probability of error is bounded. This shows that Eq. (\ref{equation_JL_lemma_correct_projection}) holds with probability $1-\delta'$. Q.E.D.
\end{proof}

It is noteworthy that Eq. (\ref{equation_random_projection_prob_success}) can be restated as:
\begin{align}
\mathbb{P}\Big(
\Big| \|f(\b{x}_i) - f(\b{x}_j)\|_2^2 - \|\b{x}_i - &\b{x}_j\|_2^2 \Big| \geq \epsilon \Big) \leq \delta.
\end{align}
This shows that there is an upper-bound on the probability of error in random projection. Therefore, after linear random projection, the distance is preserved after projection for every pair of points. 

The following lemma proves the required lower-bound on the dimensionality of subspace for random projection to work well. 
\begin{lemma}[Lower bound on dimensionality of subspace \cite{indyk1998approximate}]
The dimensionality of subspace $p$ should satisfy Eq. (\ref{equation_JL_lemma_p_lower_bound}).
\end{lemma}
\begin{proof}
The probability of error, which is Eq. (\ref{equation_prob_error_JL_lemma}), should be smaller than a constant $\delta$:
\begin{align*}
&\mathbb{P}(\text{error}) = n(n-1) e^{-(p/4) (\epsilon^2 - \epsilon^3)} \leq \delta' \\
&\implies -(p/4) (\epsilon^2 - \epsilon^3) \leq \ln(\frac{\delta'}{n(n-1)}) \\
&\implies (p/4) (\epsilon^2 - \epsilon^3) \geq \ln(\frac{n(n-1)}{\delta'}) \\
&\implies (p/4) \geq \frac{1}{\epsilon^2 - \epsilon^3} \ln(\frac{n(n-1)}{\delta'}) \approx \frac{\ln(\frac{n^2}{\delta'})}{\epsilon^2 - \epsilon^3} \\
&\implies p \geq \frac{8\ln(n) - 4 \ln(\delta')}{\epsilon^2 - \epsilon^3} = \Omega(\frac{\ln(n)}{\epsilon^2 - \epsilon^3}).
\end{align*}
Q.E.D.
\end{proof}

\begin{definition}[Concentration of Measure \cite{talagrand1996new}]\label{definition_concentration_of_measure}
Inspired by Eq. (\ref{equation_random_projection_prob_success}), we have the inequality {\citep[Lemma A.1]{plan2013one}}:
\begin{align}
\mathbb{P}\Big((1-\epsilon) \|\b{x}\|_2^2 \leq \|\b{U}^\top \b{x}\|_2^2 \leq& (1+\epsilon) \|\b{x}\|_2^2\Big) \nonumber\\
&\geq 1 - 2e^{-c \epsilon^2 p}, \label{equation_concentration_of_measure}
\end{align}
where $c>0$ is a constant, $\epsilon$ is the error tolerance, and $p$ is the dimensionality of projected data by the projection matrix $\b{U} \in \mathbb{R}^{d \times p}$.
This is the concentration of measure for a Gaussian random matrix. This inequality holds for random projection.
\end{definition}
Comparing Eqs. (\ref{equation_random_projection_prob_success}) and (\ref{equation_concentration_of_measure}) shows that random projection is equivalent to concentration of measure.

In the following, we provide a lemma which is going to be used in the proof of Lemma \ref{lemma_prob_success_whp}. Lemma \ref{lemma_prob_success_whp} provides the order of probability of success for random projection and we will use it to explain why random projection is correct (i.e., works well) with high probability. 
\begin{lemma}[\cite{dasgupta1999elementary,dasgupta2003elementary}]\label{lemma_P_beta_and_exp}
Let $p<d$ and $L := \|\b{z}_D\|_2^2$ where $\b{z}_D$ is the random projection of $\b{x}_D := \b{x}_i - \b{x}_j$ onto a $p$-dimensional subspace. If $\beta < 1$, then:
\begin{align}\label{equation_P_beta_and_exp_1}
\mathbb{P}(L \leq \beta p/d) \leq \exp\big(\frac{p}{2} (1 - \beta + \ln \beta)\big).
\end{align}
If $\beta > 1$, then:
\begin{align}\label{equation_P_beta_and_exp_2}
\mathbb{P}(L \geq \beta p/d) \leq \exp\big(\frac{p}{2} (1 - \beta + \ln \beta)\big).
\end{align}
\end{lemma}
\begin{proof}
Refer to \cite{dasgupta1999elementary,dasgupta2003elementary}.
\end{proof}

\begin{lemma}[\cite{dasgupta1999elementary,dasgupta2003elementary}]\label{lemma_prob_success_whp}
If $p$ satisfies Eq. (\ref{equation_JL_lemma_p_lower_bound}) as:
\begin{align}\label{equation_JL_lemma_p_lower_bound_2}
p \geq 4(\frac{\epsilon^2}{2} - \frac{\epsilon^3}{3})^{-1} \ln(n) = \Omega(\frac{\ln(n)}{\epsilon^2 - \epsilon^3}),
\end{align}
the local distances of points are not distorted by random projection no more than $(1 \pm \epsilon)$ with probability $\mathcal{O}(1/n^2)$. In other words, the probability of error in random projection is of the order $\mathcal{O}(1/n^2)$.
\end{lemma}
\begin{proof}
The proof is based on \cite{dasgupta2003elementary}.
\hfill\break
\textbf{Step 1 of proof:}
Consider one side of error which is $\mathbb{P}(\|f(\b{x}_i) - f(\b{x}_j)\|_2^2 \leq (1-\epsilon) (p/d) \|\b{x}_i - \b{x}_j\|_2^2)$. Take $L = \|f(\b{x}_i) - f(\b{x}_j)\|_2^2$, $\beta = (1-\epsilon)$, and $\mu = (p/d) \|\b{x}_i - \b{x}_j\|_2^2$. As $0 < \epsilon < 1$, we have $\beta<1$. By Lemma \ref{lemma_P_beta_and_exp}, we have:
\begin{align}
&\mathbb{P}(\|f(\b{x}_i) - f(\b{x}_j)\|_2^2 \leq (1-\epsilon) (p/d) \|\b{x}_i - \b{x}_j\|_2^2) \nonumber\\
&\overset{(\ref{equation_P_beta_and_exp_1})}{\leq} \exp\big(\frac{p}{2} (1 - (1-\epsilon) + \ln (1-\epsilon))\big) \nonumber\\
&= \exp\big(\frac{p}{2} (\epsilon + \ln (1-\epsilon))\big) \overset{(a)}{\leq} \exp\Big(\frac{p}{2} \big(\epsilon - (\epsilon + \frac{\epsilon^2}{2})\big)\Big) \nonumber\\
&= \exp(-\frac{p\, \epsilon^2}{4}) \overset{(b)}{\leq} \exp(-2 \ln(n)) = \frac{1}{n^2},
\end{align}
where $(a)$ is because:
\begin{align*}
\ln(1-x) \leq -(x + \frac{x^2}{2}), \quad \forall x \geq 0,
\end{align*}
and $(b)$ is because:
\begin{align*}
&p \overset{(\ref{equation_JL_lemma_p_lower_bound_2})}{\geq} 4(\frac{\epsilon^2}{2} - \frac{\epsilon^3}{3})^{-1} \ln(n) \\
&~~~~~~~~~~ \implies \exp(-\frac{p\, \epsilon^2}{4}) \leq \exp(\frac{-6 \epsilon^2}{3\epsilon^2 - 2\epsilon^3} \ln(n)) \\
&~~~~~~~~~~~~~~~~~~~\overset{(c)}{\leq} \exp(-2 \ln(n)),
\end{align*}
where $(c)$ is because:
\begin{align*}
3\epsilon^2 \geq 3\epsilon^2 - 2\epsilon^3 \implies \frac{-6 \epsilon^2}{3\epsilon^2 - 2\epsilon^3} \leq -2.
\end{align*}
\textbf{Step 2 of proof:}
Consider one side of error which is $\mathbb{P}(\|f(\b{x}_i) - f(\b{x}_j)\|_2^2 \geq (1+\epsilon) (p/d) \|\b{x}_i - \b{x}_j\|_2^2)$. Take $L = \|f(\b{x}_i) - f(\b{x}_j)\|_2^2$, $\beta = (1+\epsilon)$, and $\mu = (p/d) \|\b{x}_i - \b{x}_j\|_2^2$. As $0 < \epsilon < 1$, we have $\beta>1$. By Lemma \ref{lemma_P_beta_and_exp}, we have:
\begin{align}
&\mathbb{P}(\|f(\b{x}_i) - f(\b{x}_j)\|_2^2 \geq (1+\epsilon) (p/d) \|\b{x}_i - \b{x}_j\|_2^2) \nonumber\\
&\overset{(\ref{equation_P_beta_and_exp_2})}{\leq} \exp\big(\frac{p}{2} (1 - (1+\epsilon) + \ln (1+\epsilon))\big) \nonumber\\
&= \exp\big(\frac{p}{2} (-\epsilon + \ln (1+\epsilon))\big) \nonumber\\
&\overset{(a)}{\leq} \exp\Big(\frac{p}{2} \big(-\epsilon+\epsilon-\frac{\epsilon^2}{2} + \frac{\epsilon^3}{3} \big)\Big) \nonumber\\
&= \exp\Big(\frac{-p}{2} \big(\frac{\epsilon^2}{2} - \frac{\epsilon^3}{3} \big)\Big) \overset{(b)}{\leq} \exp(-2 \ln(n)) = \frac{1}{n^2},
\end{align}
where $(a)$ is because:
\begin{align*}
\ln(1+x) \leq x - \frac{x^2}{2} + \frac{x^3}{3}, \quad \forall x > 0,
\end{align*}
and $(b)$ is because:
\begin{align*}
&p \overset{(\ref{equation_JL_lemma_p_lower_bound_2})}{\geq} 4(\frac{\epsilon^2}{2} - \frac{\epsilon^3}{3})^{-1} \ln(n) \\
&~~~~~~~~~~ \implies -\frac{p}{2} \leq -2(\frac{\epsilon^2}{2} - \frac{\epsilon^3}{3})^{-1} \ln(n) \\
&~~~~~~~~~~ \implies -\frac{p}{2} (\frac{\epsilon^2}{2} - \frac{\epsilon^3}{3}) \leq -2 \ln(n) \\
&~~~~~~~~~~ \implies \exp\Big(\!-\frac{p}{2} (\frac{\epsilon^2}{2} - \frac{\epsilon^3}{3})\Big) \leq \exp(-2 \ln(n)).
\end{align*}
Hence, we showed in both steps 1 and 2 that the probability of every side of error is bounded by $1/n^2$. Hence, the probability of error in random projection is of the order $\mathcal{O}(1/n^2)$. 
\end{proof}

\begin{definition}[With High Probability]
When the probability of success in an algorithm goes to one by increasing the input information or a parameter or parameters of the algorithm to infinity, the algorithm is correct with high probability, denoted by w.h.p.
\end{definition}

Lemma \ref{lemma_prob_success_whp} showed that the probability of success for random projection is of order $\mathcal{O}(1/n^2)$. 
The JL Lemma, whose probability of success is of order $\mathcal{O}(1/n^2)$, is correct w.h.p. because $\mathbb{P}(\text{error}) \rightarrow 0$ if $n \rightarrow \infty$ by increasing the number of data points. In other words, the more data points we have, the more accurate the random projection is in preserving local distances after projection onto the random subspace. 
Eq. (\ref{equation_JL_lemma_correct_projection}) shows that random projection does not distort the local distances by more than a factor of $(1 \pm \epsilon)$. Therefore, with a good probability, we have:
\begin{align}
\|\b{U}^\top (\b{x}_i - \b{x}_j)\|_2^2 \approx \|\b{x}_i - \b{x}_j\|_2^2,\,\, \forall i,j \in \{1, \dots, n\}.
\end{align}

It is shown in \cite{larsen2017optimality} that for any function $f(.)$ which satisfies Eq. (\ref{equation_JL_lemma_correct_projection}), the Eq. (\ref{equation_JL_lemma_p_lower_bound_relaxed}) holds for the lower bound on the dimensionality of subspace and it does not require $f(.)$ to be necessarily random projection. This shows that JL lemma is optimal. 
Moreover, it is shown in \cite{bartal2011dimensionality} that any finite subset of Euclidean space can be embedded in a subspace with dimensionality $\mathcal{O}(\epsilon^{-2} \ln(n))$ while the distances of points are preserved with at most $(1+\epsilon)$ distortion. 

\section{Sparse Linear Random Projection}\label{section_sparse_linear_random_projection}

The JL lemma, i.e. Theorem \ref{theorem_JL_lemma}, dealt with $\ell_2$ norm. There exist some works which formulate random projection and its proofs with $\ell_1$ norm to have sparse random projection. 
In the following, we introduce the lemmas and theorems which justify sparse random projection. The proofs of these theories are omitted for brevity and they can be found in the cited papers. 

The following lemma shows that if dataset is $k$-sparse, meaning that every vector has at most $k$ non-zero values, random projection preserves distances well enough. As the dataset is sparse, random projection is sparse.
\begin{lemma}[\cite{baraniuk2006johnson}]
Consider the subset of $k$-sparse signals defined as:
\begin{align}
\mathcal{S}_k := \{\b{x} \in \mathbb{R}^d : \#\{i : |x_i| > 0\} \leq k\},
\end{align}
where $\#\{.\}$ denotes the cardinality of set. If $p = \mathcal{O}(\epsilon^{-2} k \ln(d/k))$ and $f\!: \mathbb{R}^d \rightarrow \mathbb{R}^p$, $f\!: \b{x} \mapsto \b{U}^\top \b{x}$, then with high probability, we have:
\begin{equation}\label{equation_k_sparse_signal_correctness}
\begin{aligned}
(1 - \epsilon) \|\b{x}_i - \b{x}_j\|_2^2 \leq \|&f(\b{x}_i) - f(\b{x}_j)\|_2^2 \\
&\leq (1 + \epsilon) \|\b{x}_i - \b{x}_j\|_2^2,
\end{aligned}
\end{equation}
for all $\b{x}_i, \b{x}_j \in \mathcal{S}_k$. The probability of error from Eq. (\ref{equation_k_sparse_signal_correctness}) is $2(12/\epsilon)^k e^{-(p/2) (\epsilon^2/8 - \epsilon^3/24)}$. 
\end{lemma}

In the following, we define the Restricted Isometry Property (RIP) and show that a random projection matrix whose elements are sampled from a Gaussian distribution has this property.
\begin{definition}[Restricted Isometry Property \cite{candes2006stable}]
The mapping $f: \mathbb{R}^d \rightarrow \mathbb{R}^p$, $f: \b{x} \mapsto \b{U}^\top \b{x}$ has the Restricted Isometry Property (RIP) of order $k$ and level $\epsilon \in (0,1)$ if:
\begin{align}
(1 - \epsilon) \|\b{x}\|_2^2 \leq \|\b{U}^\top \b{x}\|_2^2 \leq (1 + \epsilon) \|\b{x}\|_2^2,
\end{align}
for all $k$-sparse $\b{x} \in \mathbb{R}^d$.
\end{definition}

\begin{lemma}[\cite{candes2006stable}]
A Gaussian random matrix $\b{U} \in \mathbb{R}^{d \times p}$ has RIP if $p \geq \epsilon^{-2} k \ln(d/k)$.
\end{lemma}

\begin{theorem}[\cite{krahmer2011new}]\label{theorem_sparse_diagonal_sign_flip}
Consider the linear mapping $f: \mathbb{R}^d \rightarrow \mathbb{R}^p$, $f: \b{x} \mapsto \b{U}^\top \b{x}$ satisfying: 
\begin{align}
(1 - \epsilon) \|\b{x}\|_2^2 \leq \|\b{U}^\top \b{x}\|_2^2 \leq (1 + \epsilon) \|\b{x}\|_2^2,
\end{align}
for all $k$-sparse $\b{x} \in \mathbb{R}^d$ and suppose $\b{D}_\xi \in \mathbb{R}^{d \times d}$ is a diagonal matrix with $\xi = \pm 1$ on its diagonal. We have:
\begin{align}
\mathbb{P}\Big((1-\epsilon) \|\b{x}\|_2^2 \leq \|\b{U}^\top &\b{D}_\xi \b{x}\|_2^2 \leq (1+\epsilon) \|\b{x}\|_2^2\Big) \nonumber\\
&\geq 1 - 2e^{-c \epsilon^2 p / \ln(d)}.
\end{align}
\end{theorem}

\begin{corollary}
Comparing Theorem \ref{theorem_sparse_diagonal_sign_flip} with Definition \ref{definition_concentration_of_measure} shows that a random matrix with RIP and random column sign flips implies the JL lemma up to an $\ln(d)$ factor. 
\end{corollary}

The following lemma is going to be used in the proof of Corollary \ref{corollary_l1_meaningful_than_l2}. That corollary shows that $\ell_1$ norm makes more sense as a distance than the Euclidean distance for high-dimensional data. 
\begin{lemma}[{\citep[Theorem 2]{hinneburg2000nearest}}]\label{lemma_distance_of_points_wrt_norms}
Suppose we have $n$ i.i.d. random points in $\mathbb{R}^d$. Let $\text{dist}_{\text{min},r}$ and $\text{dist}_{\text{max},r}$ denote the minimum and maximum distances of points from the origin, respectively, with respect to a metric with norm $\ell_r$. We have:
\begin{align}
\lim_{d \rightarrow \infty} \mathbb{E}[\text{dist}_{\text{max},r} - \text{dist}_{\text{min},r}] = d^{1/r - 1/2}.
\end{align}
\end{lemma}

\begin{corollary}[\cite{hinneburg2000nearest}]\label{corollary_l1_meaningful_than_l2}
For high dimensional data, $\ell_1$ norm (such as Manhattan distance) is more meaningful for nearest neighbor comparisons than $\ell_2$ norm (such as Euclidean distance). 
\end{corollary}
\begin{proof}
Lemma \ref{lemma_distance_of_points_wrt_norms} shows that for a metric with $\ell_2$ norm (i.e. Euclidean distance), we have $\lim_{d \rightarrow \infty} \mathbb{E}[\text{dist}_{\text{max},r} - \text{dist}_{\text{min},r}] = 1$; therefore, all points are equidistant from the origin up to a constant. For a metric by $\ell_r$ norm with $r > 2$, we have $1/r - 1/2 < 0$ so the limit gets very small for large $d$ values (when $d \rightarrow \infty$). Hence, all points are equidistant from origin completely. 
For $\ell_1$ norm, we have $\lim_{d \rightarrow \infty} \mathbb{E}[\text{dist}_{\text{max},r} - \text{dist}_{\text{min},r}] = d^{1/2}$; hence, relative difference of nearest and farthest points from origin increases with dimension. Q.E.D.
\end{proof}

The above corollary showed that $\ell_1$ norm is more meaningful for nearest neighbor comparisons than $\ell_2$ norm. The below theorem extend the JL lemma to any $\ell_r$ norm with $1 \leq r \leq \infty$. 
\begin{theorem}[Extension of JL lemma to any norm {\citep[Lemma 3.1]{lee2005metric}}, \cite{ward2014dimension}]\label{theorem_linear_mapping_lr_norm}
For embedding into a subspace equipped with $\ell_r$ norm with $1 \leq r \leq \infty$, there exists a linear mapping $f: \mathbb{R}^d \rightarrow \mathbb{R}^p$, $f: \b{x} \mapsto \b{U}^\top \b{x}$, with $p \leq d$, which satisfies:
\begin{align}
\big(\frac{p}{d}\big)^{|1/r - 1/2|} \|\b{x}_i &- \b{x}_j\|_r \leq \|\b{U}^\top (\b{x}_i - \b{x}_j)\|_r \nonumber\\
&\leq \big(\frac{d}{p}\big)^{|1/r - 1/2|} \|\b{x}_i - \b{x}_j\|_r,
\end{align}
for all $\b{x}_i, \b{x}_j \in \mathbb{R}^d$.
\end{theorem}

One may ask why we often use the $\ell_2$ norm, and not $\ell_1$ norm, for random projection. This is because the following corollary shows that the suitable norm for random projection is $\ell_2$ norm and it is very hard to use other norms for random projection \cite{charikar2002dimension,brinkman2005impossibility}. 
\begin{corollary}[{\citep[Section 3]{lee2005metric}}]\label{corollary_impossible_l1_embedding}
The suitable norm for linear random projection is the $\ell_2$ norm. We have a difficulty or impossibility of dimensionality reduction with a linear map in $\ell_r$ with $r \neq 2$.
\end{corollary}
\begin{proof}
According to Theorem \ref{theorem_linear_mapping_lr_norm}, the local distances of points after projection is completely sandwiched between their local distances in the input space before projection; hence, $\|\b{U}^\top (\b{x}_i - \b{x}_j)\|_2 \approx \|\b{x}_i - \b{x}_j\|_2$. This shows that embedding into $\ell_2$ norm works very well, with constant distortion, in terms of local distance preservation. Consider $\ell_1$ as an example for $\ell_r$ with $r \neq 2$. In this case, we have $(p/d)^{1/2} \|\b{x}_i - \b{x}_j\|_1 \leq \|\b{U}^\top (\b{x}_i - \b{x}_j)\|_1 \leq (p/d)^{1/2} \|\b{x}_i - \b{x}_j\|_1$ which is not a constant distortion. Q.E.D.
\end{proof}

\begin{definition}[RIP-1 {\citep[Definition 8]{berinde2008combining}}]
For $p \geq c\, \epsilon^{-2} k \ln(d)$, there exists a linear mapping $f: \mathbb{R}^d \rightarrow \mathbb{R}^p$, $f: \b{x} \mapsto \b{U}^\top \b{x}$, with $p \leq d$, which satisfies:
\begin{align}
(1 - 2\epsilon) \|\b{x}\|_1 \leq \|\b{U}^\top \b{x}\|_1 \leq \|\b{x}\|_1,
\end{align}
for all $k$-sparse $\b{x} \in \mathbb{R}^d$, where $\epsilon$ is the error tolerance and $c$ is a constant. 
If this holds for a matrix $\b{U}$, it has the Restricted Isometry Property-1 (RIP-1).
\end{definition}

We saw that in $\ell_1$ norm, random projection is hard to work well and in $\ell_2$ norm, it is not sparse. Hence, we need some norm in between $\ell_1$ and $\ell_2$ norms to have sparse random projection which performs well enough and is not that hard to be performed. 
In the following, we define an interpolation norm. 
\begin{definition}[Interpolation Norm {\citep[Appendix]{krahmer2016unified}}]
Suppose we rearrange the elements of data $\b{x} \in \mathbb{R}^d$ from largest to smallest values. 
We divide the support of $\b{x}$ into $s$ disjoint subsets $\{\mathcal{S}_i\}_{i=1}^s$ where $\mathcal{S}_1$ corresponds to the largest values of $\b{x}$. Let $\b{x}_{\mathcal{S}_i}$ denote the values of $\b{x}$ which fall in $\mathcal{S}_i$ represented as a vector. 
The interpolation norm, denoted by $\ell_{1,2,s}$, is defined as:
\begin{align}
\|\b{x}\|_{1,2,s} := \sqrt{\sum_{i=1}^{\lceil n/s \rceil} \|\b{x}_{\mathcal{S}_i}\|_1^2}.
\end{align}
\end{definition}
Note that if $s=1$ and $s=d$, we have $\|\b{x}\|_{1,2,1} = \|\b{x}\|_2$ and $\|\b{x}\|_{1,2,d} = \|\b{x}\|_1$, respectively. Also, for any $s$, if $\b{x}$ is $k$-sparse, we have $\|\b{x}\|_{1,2,s} = \|\b{x}\|_1$. 

The following theorem is a version of JL lemma which makes use the interpolation norm. 
\begin{theorem}[{\citep[Theorem 2]{krahmer2016unified}}]\label{theorem_linear_mapping_interpolation_norm}
There exists the mapping $f: \mathbb{R}^d \rightarrow \mathbb{R}^p$, $f: \b{x} \mapsto \b{U}^\top \b{x}$ such that with probability $(1 - 2d e^{-\epsilon^2 p / s})$, we have:
\begin{align}
(0.63 - \epsilon) \|\b{x}\|_{1,2,s} \leq \|\b{U}^\top \b{x}\|_1 \leq (1.63 + \epsilon) \|\b{x}\|_{1,2,s}.
\end{align}
\end{theorem}
For $s=1$ ($\ell_2$ norm) and $s=d$ ($\ell_1$ norm), Theorem \ref{theorem_linear_mapping_interpolation_norm} is reduced to Definition \ref{definition_concentration_of_measure} (random projection) and Theorem \ref{theorem_linear_mapping_lr_norm} with $\ell_1$ norm, respectively. 
For $s$-sparse $\b{x}$, Theorem \ref{theorem_linear_mapping_interpolation_norm} is reduced to embedding into a subspace equipped with $\ell_1$ norm. 
According to \cite{krahmer2016unified}, in $s=d$ case ($\ell_1$ norm) in Theorem \ref{theorem_linear_mapping_interpolation_norm}, we have $p \geq c d \ln(d)$ where $c$ is a constant. This contradicts $p \leq d$ so this coincides with Corollary \ref{corollary_impossible_l1_embedding}. 
However, for the $s$-sparse data in Theorem \ref{theorem_linear_mapping_interpolation_norm}, we need $p = \mathcal{O}(s \ln(d)) < \mathcal{O}(d)$ \cite{krahmer2016unified}. This shows that we can embed data onto a subspace equipped with $\ell_1$ norm, using linear projection, if the data are sparse enough. 

It is shown in \cite{baraniuk2006johnson} that the sparse random projection and JL lemma are related to compressed sensing \cite{donoho2006compressed}.
Some other works on sparse random projection are sparse Cauchy random projections \cite{li2007nonlinear,ramirez2012reconstruction}, Bernoulli random projection \cite{achlioptas2003database}, and very sparse random projection \cite{li2006very}. 
The sparse random projection can also be related to random projection on hypercubes {\citep[Chapter 7.2]{vempala2005random}}. This will be explained in Section \ref{section_random_projection_hypercube}. 

\section{Applications of Linear Random Projection}\label{section_applications_of_random_projection}

There exists various applications for random projection. Two moet well-known applications are low-rank matrix approximation and approximate nearest neighbor search which we explain in the following. 

\subsection{Low-Rank Matrix Approximation Using Random Projection}


One of the applications of random projection is low-rank approximation of a matrix which we explain in the following.
According to the Eckart-Young-Mirsky theorem \cite{eckart1936approximation}, Singular Value Decomposition (SVD) can be used for low-rank approximation of a matrix. Consider the matrix $\b{X} \in \mathbb{R}^{d \times n}$ where $d \geq n$. 
The time complexity for SVD of this matrix $\mathcal{O}(d n^2)$ which grows if the matrix is large. We can improve the time complexity to $\mathcal{O}(d n \ln(n))$ using random projections \cite{papadimitriou2000latent}. In this technique, the low-rank approximation is not completely optimal but good enough. 
This method has two steps: First, it finds a smaller matrix $\b{Y}$ by random projection:
\begin{align}\label{equation_lowRankApprox_Y_projection}
\mathbb{R}^{p \times n} \ni \b{Y} := \frac{1}{\sqrt{p}} \b{U}^\top \b{X},
\end{align}
where $\b{U} \in \mathbb{R}^{d \times p}$ is the random projection matrix whose elements are independently drawn from standard normal distribution. Note that $p$ should satisfy Eq. (\ref{equation_JL_lemma_p_lower_bound}); for example, $p \geq c \ln(n) / \epsilon^2$ where $c$ is a positive constant.
As we have $p \ll d$, the matrix $\b{Y}$ is much smaller than the matrix $\b{X}$. We calculate the SVD of $\b{Y}$:
\begin{align}\label{equation_lowRankApprox_Y}
\b{Y} = \b{A} \b{\Lambda} \b{B}^\top = \sum_{i=1}^p \lambda_i \b{a}_i \b{b}_i^\top,
\end{align}
where $\b{A} = [\b{a}_1, \dots, \b{a}_p] \in \mathbb{R}^{p \times p}$ and $\b{B} = [\b{b}_1, \dots, \b{b}_p] \in \mathbb{R}^{n \times p}$ are the matrices of singular vectors and $\b{\Lambda} = \textbf{diag}([\lambda_1, \dots, \lambda_p]^\top) \in \mathbb{R}^{p \times p}$ contains the singular values. 
Note that the SVD of $\b{Y}$ is much faster than the SVD of $\b{X}$ because of the smaller size of matrix. 
The matrix $\b{X}$ can be approximated as its projection as:
\begin{align}\label{equation_low_rank_approx_X}
\mathbb{R}^{d \times n} \ni \widetilde{\b{X}}_p \approx \b{X} \Big(\sum_{i=1}^p \b{b}_i \b{b}_i^\top\Big).
\end{align}
Note that the right singular values $\b{b}_i$'s, which are used here, are equal to the eigenvectors of $\b{Y}^\top \b{Y}$ \cite{ghojogh2019unsupervised}. 
The rank of Eq. (\ref{equation_low_rank_approx_X}) is $p$ because $p \ll d, n$; hence, it is a low-rank approximation of $\b{X}$. 
In the following, we show why this is a valid low-rank approximation.

\begin{lemma}[{\citep[Lemma 3]{papadimitriou2000latent}}]
Let the SVD of $\b{X}$ be $\b{X} = \b{C} \b{\Sigma} \b{E}^\top = \sum_{i=1}^{d} \sigma_i \b{c}_i \b{e}_i^\top$ where $\b{C} = [\b{c}_1, \dots, \b{c}_{d}] \in \mathbb{R}^{d \times d}$, $\b{E} = [\b{e}_1, \dots, \b{e}_{d}] \in \mathbb{R}^{n \times d}$, and $\b{\Sigma} = \textbf{diag}([\sigma_1, \dots, \sigma_{d}]^\top) \in \mathbb{R}^{d \times d}$ are the left singular vectors, right singular vectors, and singular values, respectively. 
Also, let the SVD of $\b{X}$ with top $p$ singular values be $\b{X}_p := \sum_{i=1}^{p} \sigma_i \b{c}_i \b{e}_i^\top$. 
If $p \geq c \ln(n) / \epsilon^2$, the singular values of $\b{Y}$ are not much smaller than the singular values of $\b{X}$, i.e.:
\begin{align}\label{equation_lowRankApprox_lambda_squared_geq}
\sum_{i=1}^p \lambda_i^2 \geq (1 - \epsilon) \sum_{i=1}^p \sigma_i^2 = (1 - \epsilon) \|\b{X}_p\|_F^2,
\end{align}
where $\|.\|_F$ is the Frobenius norm. 
\end{lemma}
\begin{proof}
Refer to {\citep[Appendix]{papadimitriou2000latent}}.
\end{proof}

\begin{theorem}[{\citep[Theorem 5]{papadimitriou2000latent}}]
The low-rank approximation $\widetilde{\b{X}}$ approximates $\b{X}$ well enough:
\begin{align}
\|\b{X} - \widetilde{\b{X}}_p\|_F^2 \leq \|\b{X} - \b{X}_p\|_F^2 + 2 \epsilon \|\b{X}_p\|_F^2.
\end{align}
\end{theorem}
\begin{proof}
Consider the complete SVD with $\b{A} = [\b{a}_1, \dots, \b{a}_p] \in \mathbb{R}^{p \times p}$, $\b{B} = [\b{b}_1, \dots, \b{b}_{n}] \in \mathbb{R}^{n \times n}$, and $\b{\Lambda} \in \mathbb{R}^{p \times n}$.
We have:
\begin{align}
&\|\b{X} - \widetilde{\b{X}}_p\|_F^2 \overset{(a)}{=} \sum_{i=1}^{n} \|(\b{X} - \widetilde{\b{X}}_p) \b{b}_i\|_2^2 \nonumber\\
&= \sum_{i=1}^{n} \|\b{X} \b{b}_i - \widetilde{\b{X}}_p \b{b}_i \|_2^2 
\nonumber\\
&\overset{(\ref{equation_low_rank_approx_X})}{=} \sum_{i=1}^{n} \|\b{X} \b{b}_i - \b{X} \big(\sum_{i=1}^p \b{b}_i \underbrace{\b{b}_i^\top\big) \b{b}_i}_{=1}\|_2^2 \overset{(b)}{=} \sum_{i=p+1}^{n} \|\b{X} \b{b}_i\|_2^2 \nonumber\\
&= \|\b{X}\|_F^2 - \sum_{i=1}^p \|\b{X} \b{b}_i\|_2^2, \label{equation_lowRankApprox_X_minus_Xtilde}
\end{align}
where $(a)$ and $(b)$ are because the singular vectors $\{\b{a}_i\}$ and $\{\b{b}_i\}$ are orthonormal. 
Similarly, we can show:
\begin{align}
\|\b{X} - \b{X}_p\|_F^2 = \|\b{X}\|_F^2 - \|\b{X}_p\|_F^2. \label{equation_lowRankApprox_X_minus_Xp}
\end{align}
From Eqs. (\ref{equation_lowRankApprox_X_minus_Xtilde}) and (\ref{equation_lowRankApprox_X_minus_Xp}), we have:
\begin{align}\label{equation_lowRankApprox_X_minus_Xp_tilde_equals_sth}
\|\b{X} - \widetilde{\b{X}}_p\|_F^2 = \|\b{X} - \b{X}_p\|_F^2 + \|\b{X}_p\|_F^2 - \sum_{i=1}^p \|\b{X} \b{b}_i\|_2^2.
\end{align}
We also have:
\begin{align}
&\sum_{i=1}^p \|\b{Y} \b{b}_i\|_2^2 = \sum_{i=1}^p \b{b}_i^\top \b{Y}^\top \b{Y} \b{b}_i \nonumber\\
&\overset{(\ref{equation_lowRankApprox_Y})}{=} \sum_{i=1}^p \b{b}_i^\top \sum_{j=1}^p \lambda_j \b{b}_j \b{a}_j^\top \sum_{k=1}^p \lambda_k \b{a}_k \b{b}_k^\top \b{b}_i \nonumber\\
&= \sum_{i=1}^p \sum_{j=1}^p \sum_{k=1}^p \lambda_j \lambda_k \underbrace{\b{b}_i^\top \b{b}_j}_{=\delta_{ij}} \underbrace{\b{a}_j^\top \b{a}_k}_{=\delta_{jk}} \underbrace{\b{b}_k^\top \b{b}_i}_{=\delta_{ki}} \overset{(a)}{=} \sum_{i=1}^p \lambda_i^2, \label{equation_lowRankApprox_sum_lambda_squared}
\end{align}
where $\delta_{ij}$ is the Kronecker delta and $(a)$ is because the singular vectors $\{\b{a}_i\}$ and $\{\b{b}_i\}$ are orthonormal. 
On the other hand, we have:
\begin{align}
&\sum_{i=1}^p \|\b{Y} \b{b}_i\|_2^2 \overset{(\ref{equation_lowRankApprox_Y_projection})}{=} \sum_{i=1}^p \|\frac{1}{\sqrt{p}} \b{U}^\top \b{X} \b{b}_i\|_2^2 \nonumber \\
&\overset{(\ref{equation_lowRankApprox_sum_lambda_squared})}{\implies} \lambda_i^2 = \sum_{i=1}^p \|\frac{1}{\sqrt{p}} \b{U}^\top \b{X} \b{b}_i\|_2^2. \label{equation_lowRankApprox_lambda_squared}
\end{align}
According to Eq. (\ref{equation_JL_lemma_correct_projection}), with high probability, we have:
\begin{align}\label{equation_lowRankApprox_JL_inequality}
(1-\epsilon) \|\b{X}\b{b}_i\|_2^2 \leq \|\frac{1}{\sqrt{p}} \b{U}^\top \b{X} \b{b}_i\|_2^2 \leq (1+\epsilon) \|\b{X}\b{b}_i\|_2^2.
\end{align}
From Eqs. (\ref{equation_lowRankApprox_lambda_squared}) and (\ref{equation_lowRankApprox_JL_inequality}), with high probability, we have:
\begin{align*}
&\lambda_i^2 \leq (1+\epsilon) \|\b{X}\b{b}_i\|_2^2 \\
&\implies \sum_{i=1}^p (1+\epsilon) \|\b{X}\b{b}_i\|_2^2 \geq \sum_{i=1}^p \lambda_i^2 \overset{(\ref{equation_lowRankApprox_lambda_squared_geq})}{\geq} (1 - \epsilon) \|\b{X}_p\|_F^2 \\
&\implies \sum_{i=1}^p \|\b{X}\b{b}_i\|_2^2 \geq \frac{1-\epsilon}{1+\epsilon} \|\b{X}_p\|_F^2 \geq (1-2\epsilon) \|\b{X}_p\|_F^2 \\
&\implies \|\b{X}_p\|_F^2 - \sum_{i=1}^p \|\b{X} \b{b}_i\|_2^2 \leq 2\epsilon\, \|\b{X}_p\|_F^2 \\
&\overset{(\ref{equation_lowRankApprox_X_minus_Xp_tilde_equals_sth})}{\implies} \|\b{X} - \widetilde{\b{X}}_p\|_F^2 \leq \|\b{X} - \b{X}_p\|_F^2 + 2\epsilon\, \|\b{X}_p\|_F^2.
\end{align*}
Q.E.D.
\end{proof}

\begin{lemma}[\cite{papadimitriou2000latent}]
The time complexity of low-rank approximation (with rank $p$) of matrix $\b{X}$ using random projection is $\mathcal{O}(d n \ln(n))$ which is much better than the complexity of SVD on $\b{X}$ which is $\mathcal{O}(d n^2)$. 
\end{lemma}
\begin{proof}
Complexities of Eqs. (\ref{equation_lowRankApprox_Y_projection}) and (\ref{equation_lowRankApprox_Y}) are $\mathcal{O}(d n p)$ and $\mathcal{O}(n p^2)$, respectively, noticing that $p \ll n \leq d$. We also have $p \geq c \ln(n) / \epsilon^2$ where $c$ is a constant. Hence, the total complexity is $\mathcal{O}(d n p) + \mathcal{O}(n p^2) \in \mathcal{O}(d n p) = \mathcal{O}(d n \ln(n))$. Q.E.D.
\end{proof}

Some other works, such as \cite{frieze2004fast,martinsson2007fast}, have improved the time complexity of using random projection for low-rank applications.
For more information on using random projection for low-rank approximation, one can refer to {\citep[Chapter 8]{vempala2005random}}. 

\subsection{Approximate Nearest Neighbor Search}

\subsubsection{Random Projection onto Hypercube}\label{section_random_projection_hypercube}

Let $\mathbb{Z}^d_w$ denote the set of $d$-dimensional integer vectors with $w$ possible values; for example, $\mathbb{Z}^d_2 := \{0,1\}^d$ is the set of binary values. Also let $\|\b{x}_i - \b{x}_j\|_H$ denote the Hamming distance between two binary vectors $\b{x}_i$ and $\b{x}_j$. We can have random projection of a binary dataset onto a hypercube $\mathbb{Z}^d_w$ where $\ell_1$ norm or Hamming distances are used \cite{kushilevitz1998algorithm,kushilevitz2000efficient}. 

\begin{theorem}[Random projection onto hypercube \cite{kushilevitz1998algorithm,kushilevitz2000efficient}]\label{theorem_random_projection_hypercube}
Let $\text{mod}$ denote the modulo operation.
Consider a binary vector $\b{x} \in \mathbb{Z}^d_2$ which is projected as $f(\b{x}) = \b{U}^\top \b{x} \text{ mod } 2$ with a random binary projection matrix $\mathbb{R} \in \{0,1\}^{d \times p}$, where $p = O(\ln(n) / \epsilon^2)$ (usually $p \ll d$). The elements of $\b{U}$ are i.i.d. with Bernoulli distribution having probability $\xi = (\epsilon^2 / \ell)$ to be one and probability $(1-\xi)$ to be zero. We have:
\begin{align}
& \text{if }\,\, \|\b{x}_i - \b{x}_j\|_H < \frac{\ell}{4} \nonumber\\
&\quad\quad\quad\implies \|f(\b{x}_i) - f(\b{x}_j)\|_H < (1+\epsilon)p\,\xi\,\frac{\ell}{4}, \\
& \text{if }\,\, \frac{\ell}{4} \leq \|\b{x}_i - \b{x}_j\|_H \leq \frac{\ell}{2\epsilon} \nonumber\\
&\implies (1-\epsilon)p\,\xi\leq\frac{\|f(\b{x}_i) - f(\b{x}_j)\|_H}{\|\b{x}_i - \b{x}_j\|_H} < (1+\epsilon)p\,\xi, \\
& \text{if }\,\, \|\b{x}_i - \b{x}_j\|_H > \frac{\ell}{2\epsilon} \nonumber\\
&\quad\quad\quad\implies \|f(\b{x}_i) - f(\b{x}_j)\|_H > (1-\epsilon)p\,\xi\,\frac{\ell}{2\epsilon},
\end{align}
for all $\b{x}_i, \b{x}_j \in \mathbb{Z}^d_2$, with probability at least $(1-e^{-c \epsilon^4 p})$ where $c$ is a positive constant. 
\end{theorem}
\begin{proof}
Proof is available in \cite{vempala2005random,liberty2006random}.
\end{proof}


\subsubsection{Approximate Nearest Neighbor Search by Random Projection}

Consider a dataset $\mathcal{X} := \{\b{x}_i \in \mathbb{R}^d\}_{i=1}^n$. 
The nearest neighbor search problem refers to finding the closest point of dataset $\b{x}^* \in \mathcal{X}$ to a query point $\b{q} \in \mathbb{R}^d$.
One solution to this problem is to calculate the distances of the query point from all $n$ points of dataset and return the point with the smallest distance. However, its time and space complexities are both $\mathcal{O}(nd)$ which are not good. There is an algorithm for nearest neighbor search \cite{meiser1993point} with time complexity $\mathcal{O}(\text{poly}(d, \ln(n)))$ where $\text{ploy}()$ is a polynomial combination of its inputs. However, the space complexity of this algorithm is $\mathcal{O}(n^d)$. 

To have better time and space complexities, approximate nearest neighbor search is used. Approximate nearest neighbor search returns a point $\b{x}^* \in \mathcal{X}$ which satisfies:
\begin{align}
\|\b{q} - \b{x}^*\|_2 \leq (1 + \epsilon)\, (\min_{x \in \mathcal{X}} \|\b{q} - \b{x}\|_2),
\end{align}
where $\epsilon >0$ is the error tolerance. This problem is named $\epsilon$-approximate nearest neighbor search problem \cite{andoni2006near}.
We can relax this definition if we take the acceptable distance $r$ from the user:
\begin{align}\label{equation_approx_knn_with_r}
\|\b{q} - \b{x}^*\|_2 \leq (1 + \epsilon)\, r.
\end{align}
If no such point is found in $\mathcal{X}$, null is returned. 
This relaxation is valid because the smallest $r$ can be found using a binary search whose number of iterations is constant with respect to $d$ and $n$. 

We can use the introduced random projection onto hypercube for $\epsilon$-approximate nearest neighbor search \cite{kushilevitz1998algorithm,kushilevitz2000efficient}. 
We briefly explain this algorithm in the following. 
Assume the dataset is binary, i.e., $\mathcal{X} := \{\b{x}_i \in \mathbb{Z}_2^d\}_{i=1}^n$. If not, the values of vector elements are quantized to binary strings and the values of a vector are reshaped to become a binary vector. 
Let the dimensionality of binary (or quantized and reshaped) vectors is $d$. 
Recall the binary search required for finding the smallest distance $r$ in Eq. (\ref{equation_approx_knn_with_r}). 
In the binary search for finding the smallest distance $r$, we should try several distances. For every distance, we perform $k$ independent random projections onto $k$ hypercubes and then we select one of these random projections randomly. In every random projection, the points $\mathcal{X}$ are projected onto a random $p$-dimensional hypercube to have $\{f(\b{x}_i) \in \mathbb{Z}_2^p\}_{i=1}^n$. 
In the low-dimensional projected subspace, the comparison of points is very fast because (I) the subspace dimensionality $p$ is much less than the original dimensionality $d$ and (II) calculation of Hamming distance is faster and easier than Euclidean distance. Therefore, random projections onto hypercubes are very useful for approximate nearest neighbor search. 

\begin{lemma}
The above algorithm for approximate nearest neighbor search is correct with probability $(1 - \delta)$ where $0< \delta \ll 1$. 
The time complexity of the above algorithm is $\mathcal{O}(\frac{d}{\epsilon^4} \ln(\frac{n}{\delta}) \ln(d))$ and its space complexity is $\mathcal{O}(d^2 (c_1 n \ln(d))^{c_2 / \epsilon^4})$ where $c_1$ and $c_2$ are constants. 
\end{lemma}
\begin{proof}
Proof is available in \cite{kushilevitz1998algorithm,kushilevitz2000efficient,vempala2005random,liberty2006random}.
\end{proof}


Random projection for approximate nearest neighbor search is also related to hashing; for example, see locality sensitive hashing \cite{slaney2008locality}.
There exist some other works on random projection for approximate nearest neighbor search \cite{indyk1998approximate,ailon2006approximate}. 
For more information on using random projection for approximate nearest neighbor search, one can refer to {\citep[Chapter 7]{vempala2005random}}.

\section{Random Fourier Features and Random Kitchen Sinks for Nonlinear Random Projection}\label{section_RFF_and_RKS}

So far, we explained linear random projection in which a linear projection is used. We can have nonlinear random projection which is a much harder task to analyze theoretically. A nonlinear random projection can be modeled as a linear random projection followed by a nonlinear function. Two fundamental works on nonlinear random projection are RFF \cite{rahimi2007random} and RKS \cite{rahimi2008weighted}, explained in the following. 

\subsection{Random Fourier Features for Learning with Approximate Kernels}

When the pattern of data is nonlinear, wither a nonliner algorithm should be used or the nonlinear data should be transformed using kernels to be able to use the linear methods for nonlinear patterns \cite{ghojogh2021reproducing}. 
Computation of kernels is a time-consuming task because the points are pulled to the potentially high dimensional space and then the inner product of pulled points to the Reproducing Kernel Hilbert Space (RKHS) are calculated (see \cite{ghojogh2021reproducing} for more details). 
Random Fourier Features (RFF) are used for accelerating kernel methods \cite{rahimi2007random}. 
For this, RFF transforms data to a low-dimensional space in contrast to the kernels which transform data to a potentially high dimensional space. RFF approximates the kernel, which is the inner product of pulled data to RKHS, by inner product of low-dimensional feature maps of data. This feature map is a random feature map and is $z: \mathbb{R}^d \rightarrow \mathbb{R}^{2p}$ where $p \ll d$; therefore, its computation of inner products is much faster than computing the inner product of data in RKHS. It satisfies the following approximation:
\begin{align}\label{equation_RFF_kernel_approximation}
k(\b{x}, \b{y}) = \b{\phi}(\b{x})^\top \b{\phi}(\b{y}) \approx z(\b{x})^\top z(\b{y}),
\end{align}
where $k(.,.)$ denotes the kernel function and $\b{\phi}(.)$ is the pulling function to RKHS. 
Note that $z$ is a nonlinear mapping which can be seen as a nonlinear random projection. It is a linear random projection $f : \mathbb{R}^d \rightarrow \mathbb{R}^{p}$ followed by nonlinear sine and cosine functions. We will see the formulation of $z$ and $f$ functions later. 

RFF works with positive definite kernels which are shift-invariant (also called stationary kernels), i.e., $k(\b{x}, \b{y}) = k(\b{x} - \b{y})$ \cite{ghojogh2021reproducing}.
Consider the inverse Fourier transform of the kernel function $k$:
\begin{equation}\label{equation_kernel_inverse_Fourier_transform}
\begin{aligned}
k(\b{x} - \b{y}) = \int_{\mathbb{R}^d} \widehat{k}(\b{u}) e^{j \b{u}^\top (\b{x} - \b{y})} d\b{u} \\
\overset{(a)}{=} \mathbb{E}_{\b{u}}[\zeta_{\b{u}}(\b{x}) \zeta_{\b{u}}(\b{y})^*],
\end{aligned}
\end{equation}
where $\b{u} \in \mathbb{R}^d$ is the frequency, $j$ is the imaginary unit, the superscript $^*$ denotes conjugate transpose, $\mathbb{E}[.]$ is the expectation operator, $(a)$ is because we define $\mathbb{R} \ni \zeta_{\b{u}}(\b{x}) := e^{j \b{u}^\top \b{x}}$, and $\widehat{k}(\b{u})$ is the Fourier transform of kernel function:
\begin{align}\label{equation_kernel_Fourier_transform}
\widehat{k}(\b{u}) = \frac{1}{2\pi} \int_{\mathbb{R}^d} e^{-j \b{u}^\top \b{x}}\, k(\b{x})\, d\b{x}. 
\end{align}
In Eq. (\ref{equation_kernel_inverse_Fourier_transform}), the kernel function $k$ and the transformed kernel $\widehat{k}$ are real valued so the sine part of $e^{j \b{u}^\top (\b{x} - \b{y})}$ can be ignored (see Euler's equation) and can be replaced with $\cos(\b{u}^\top (\b{x} - \b{y}))$; hence, $\zeta_{\b{u}}(\b{x}) = \cos(\b{u}^\top \b{x})$. Let:
\begin{align}\label{equation_RFF_f_w}
\mathbb{R}^2 \ni z_u(\b{x}) := [\cos(\b{u}^\top \b{x}), \sin(\b{u}^\top \b{x})]^\top.
\end{align}
In Eq. (\ref{equation_kernel_inverse_Fourier_transform}), $\widehat{k}(\b{u})$ can be seen as a $d$-dimensional probability density function. 
We draw $p$ i.i.d. random projection vectors, $\{\b{u}_t \in \mathbb{R}^d\}_{t=1}^p$, from $\widehat{k}(\b{u})$. 
We define a normalized vector version of Eq. (\ref{equation_RFF_f_w}) as:
\begin{equation}\label{equation_RFF_z_vector}
\begin{aligned}
\mathbb{R}^{2p} \ni z(\b{x}) := \frac{1}{\sqrt{p}} [&\cos(\b{u}_1^\top \b{x}), \dots, \cos(\b{u}_p^\top \b{x}), \\
&\sin(\b{u}_1^\top \b{x}), \dots, \sin(\b{u}_p^\top \b{x})]^\top.
\end{aligned}
\end{equation}
If we take linear random projection as:
\begin{align}
\mathbb{R}^{p} \ni f(\b{x}) := \b{U}^\top \b{x} = [\b{u}_1^\top \b{x}, \dots, \b{u}_p^\top \b{x}]^\top,
\end{align}
with the projection matrix $\b{U} := [\b{u}_1, \dots, \b{u}_p] \in \mathbb{R}^{d \times p}$, we see that the function $z(.)$ is applying sine and cosine functions to a linear random projection $f(.)$ of data $\b{x}$.

According to $\cos(a-b) = \cos(a) \cos(b) + \sin(a) \sin(b)$, we have:
\begin{align}\label{equation_RFF_f_innser_products}
z(\b{x})^\top z(\b{y}) = \cos(\b{u}^\top (\b{x} - \b{y})).
\end{align}
Considering $\zeta_{\b{u}}(\b{x}) = \cos(\b{u}^\top \b{x})$, we have:
\begin{align}
k(\b{x}, \b{y}) &= k(\b{x} - \b{y}) \overset{(\ref{equation_kernel_inverse_Fourier_transform})}{=} \mathbb{E}_{\b{u}}[\zeta_{\b{u}}(\b{x}) \zeta_{\b{u}}(\b{y})^*] \nonumber \\
&= \mathbb{E}_{\b{u}}[z(\b{x})^\top z(\b{y})]. \label{equation_RFF_kernel_approx_expectation}
\end{align}
Thus, Eq. (\ref{equation_RFF_kernel_approximation}) holds in expectation. 
We show in the following that the variance of Eq. (\ref{equation_RFF_kernel_approx_expectation}) is small to have Eq. (\ref{equation_RFF_kernel_approximation}) as a valid approximation. 
According to Eq. (\ref{equation_RFF_f_innser_products}), we have $-1 \leq z(\b{x})^\top z(\b{y}) \leq 1$; hence, we can use the Hoeffding's inequality \cite{hoeffding1963probability}:
\begin{align}\label{equation_RFF_Hoeffding}
\mathbb{P}\big(|z(\b{x})^\top z(\b{y}) - k(\b{x}, \b{y})| \geq \epsilon\big) \leq 2 \exp(-\frac{p\,\epsilon^2}{2}),
\end{align}
which is a bound on the probability of error in Eq. (\ref{equation_RFF_kernel_approximation}) with error tolerance $\epsilon \geq 0$. 
Therefore, the approximation in Eq. (\ref{equation_RFF_kernel_approximation}) holds and we can approximate kernel computation with inner product in a random lower dimensional subspace. This improves the speed of kernel machine learning methods because $p$ is much less than both the dimensionality of input space, i.e. $d$, and the dimensionality of RKHS. 
Although the Eq. (\ref{equation_RFF_Hoeffding}) proves the correctness of approximation in RFF, the following theorem shows more strongly that the approximation in RFF is valid. 

\begin{theorem}[{\citep[Claim 1]{rahimi2007random}}]
Let $\mathcal{M}$ be a compact subset of $\mathbb{R}^d$ with diameter $\text{diam}(\mathcal{M})$. We have:
\begin{align}
&\mathbb{P}\Big(\sup_{\b{x}, \b{y} \in \mathcal{M}} |z(\b{x})^\top z(\b{y}) - k(\b{x}, \b{y})| \geq \epsilon\Big) \nonumber\\
&~~~~~~~~~\leq 2^8 \big(\frac{\sigma\, \text{diam}(\mathcal{M})}{\epsilon}\big)^2 \exp\big(\!-\frac{p\,\epsilon^2}{4(d+2)}\big),
\end{align}
\end{theorem}
where $\sigma^2 := \mathbb{E}_{\b{u}}[\b{u}^\top \b{u}]$ with probability density function $\widehat{k}(\b{u})$ used in expectation. The required lower bound on the dimensionality of random subspace is:
\begin{align}
p \geq \Omega\Big(\frac{d}{\epsilon^2} \ln\big(\frac{\sigma\, \text{diam}(\mathcal{M})}{\epsilon}\big)\Big).
\end{align}
\begin{proof}
Proof is available in {\citep[Appendix A]{rahimi2007random}}.
\end{proof}

In summary, the algorithm of RFF is as follows. Given a positive definite stationary kernel function $k$, we calculate its Fourier transform by Eq. (\ref{equation_kernel_Fourier_transform}). Then, we treat this Fourier transform as a probability density function and draw $p$ i.i.d. random projection vectors $\{\b{u}_t \in \mathbb{R}^d\}_{t=1}^p$ from that. For every point $\b{x}$, the mapping $z(\b{x})$ is calculated using Eq. (\ref{equation_RFF_z_vector}). Then, for every two points $\b{x}$ and $\b{y}$, their kernel is approximated by Eq. (\ref{equation_RFF_kernel_approximation}). The approximated kernel can be used in kernel machine learning algorithms. 
As was explained, RFF can be interpreted as nonlinear random projection because it applied nonlinear sine and cosine functions to linear random projections. 

Finally, it is noteworthy that there also exists another method for kernel approximation, named random binning features, which we refer the interested reader to \cite{rahimi2007random} for information about it. Similar to RFF, random binning features can also be interpreted as nonlinear random projection. 

It is shown in \cite{yang2012nystrom} that there is a connection between the Nystr{\"o}m method \cite{ghojogh2021reproducing} and RFF. 
Some other works on RFF are \cite{chitta2012efficient,sutherland2015error,sriperumbudur2015optimal,li2019towards}.

\subsection{Random Kitchen Sinks for Nonlinear Random Projection}


Random Kitchen Sinks (RKS) \cite{rahimi2008weighted} was proposed after development of RFF \cite{rahimi2007random}. 
RKS is a nonlinear random projection where a nonlinear function is applied to a linear random projection. 
RKS models this nonlinear random projection as a random layer of neural network connecting $d$ neurons to $p$ neurons. 
A random network layer is actually a linear random projection followed by a nonlinear activation function:
\begin{align}\label{equation_RKS_layer}
g(\b{x}) := \phi\Big(\sum_{j=1}^d \sum_{t=1}^p u_{jt}\, x_j\Big) \overset{(\ref{equation_linear_random_projection})}{=} \phi(\b{U}^\top \b{x}) = \phi\big(f(\b{x})\big),
\end{align}
where $g(.)$ is the function representing the whole layer, $\phi(.)$ is the possibly nonlinear activation function, $f(.)$ is the linear random projection, $\b{U} \in \mathbb{R}^{d \times p}$ is the matrix of random weights of network layer, $u_{jt}$ is the $(j,t)$-th element of $\b{U}$, and $x_j$ is the $j$-th element of $\b{x} \in \mathbb{R}^d$. 
Therefore, a random network layer can be seen as a nonlinear random projection.
According to the universal approximation theorem \cite{hornik1989multilayer,huang2006universal}, the Eq. (\ref{equation_RKS_layer}) can fit any possible decision function mapping continuous inputs to a finite set of classes to any desired level of accuracy.  
According to the representer theorem \cite{aizerman1964theoretical}, we can state the function $g(.)$ in RKHS as \cite{ghojogh2021reproducing}:
\begin{align}\label{equation_RKS_representer_theorem}
g(\b{x}) = \sum_{i=1}^\infty \alpha_i\, \phi(\b{x}; w_i),
\end{align}
where $\{w_i\} \in \Omega$ are the parameters ($\Omega$ is the space of parameters), $\phi(.)$ is the pulling function to RKHS, $\{\alpha_i\}$ are the weights. 
Comparing Eqs. (\ref{equation_RKS_layer}) and (\ref{equation_RKS_representer_theorem}) shows that there is a connection between RKS and kernels \cite{rahimi2008uniform} if we consider the pulling function $\phi(.)$ as the activation function, the parameters $\{w_i\}$ as the parameters of activation function, and the weights $\{\alpha_i\}$ as the weights of network layer. 

Consider a classification task where $l_i$ is the label of $\b{x}_i$. The empirical risk between the labels and the output of network layer is:
\begin{align}\label{equation_RKS_empirical_risk}
R_e(g) := \frac{1}{n} \sum_{i=1}^n \ell\big(g(\b{x}_i), l_i\big),
\end{align}
where $\ell$ is a loss function. 
The true risk is:
\begin{align}\label{equation_RKS_true_risk}
R_t(g) := \mathbb{E}_{\{\b{x}_i, l_i\}}\Big[ \ell\big(g(\b{x}_i), l_i\big)\Big],
\end{align}
where $\mathbb{E}[.]$ is the expectation operator. 
We want to minimize the empirical risk; according to Eq. (\ref{equation_RKS_representer_theorem}), we have:
\begin{align}\label{equation_RKS_empirical_risk_minimize_twoVars}
\underset{\{w_i, \alpha_i\}_{i=1}^p}{\text{minimize}}\quad R_e\Big(\sum_{i=1}^p \alpha_i\, \phi(\b{x}; w_i)\Big),
\end{align}
which is a joint minimization over $\{w_i\}_{i=1}^p$ and $\{\alpha_i\}_{i=1}^p$. 
In RKS, we sample the parameters of activation function, $\{w_i\}_{i=1}^p$, randomly from a distribution $\mathbb{P}(w)$ in space $\Omega$. Then, we minimize the empirical risk over only the network weights $\{\alpha_i\}_{i=1}^p$. In other words, we eliminate the optimization variables $\{w_i\}_{i=1}^p$ by their random selection and minimize over only $\{\alpha_i\}_{i=1}^p$. In practice, the distribution $\mathbb{P}(w)$ can be any distribution in the space of parameters of activation functions. 
Let $\b{\alpha} := [\alpha_1, \dots, \alpha_p]^\top \in \mathbb{R}^p$. If $\{w_i\}_{i=1}^p$ are randomly sampled for all $i \in \{1, \dots, n\}$, suppose $\b{g}_i := [\phi(\b{x}_i, \b{w}_1), \dots, \phi(\b{x}_i, \b{w}_p)]^\top \in \mathbb{R}^p, \forall i$. 
According to Eqs. (\ref{equation_RKS_representer_theorem}), (\ref{equation_RKS_empirical_risk}), and (\ref{equation_RKS_empirical_risk_minimize_twoVars}), the optimization in RKS is:
\begin{equation}\label{equation_RKS_optimziation}
\begin{aligned}
& \underset{\b{\alpha} \in \mathbb{R}^p}{\text{minimize}}
& & \frac{1}{n} \sum_{i=1}^n \ell\big(\b{\alpha}^\top \b{g}_i, l_i\big) \\
& \text{subject to}
& & \|\b{\alpha}\|_\infty \leq \frac{c}{p},
\end{aligned}
\end{equation}
where $\|.\|_\infty$ is the maximum norm and $c$ is a positive constant. In practice, RKS relaxes the constraint of this optimization to a quadratic regularizer: 
\begin{equation}\label{equation_RKS_optimziation_practice}
\begin{aligned}
& \underset{\b{\alpha} \in \mathbb{R}^p}{\text{minimize}}
& & \frac{1}{n} \sum_{i=1}^n \ell\big(\b{\alpha}^\top \b{g}_i, l_i\big) + \lambda\|\b{\alpha}\|_2^2,
\end{aligned}
\end{equation}
where $\lambda >0$ is the regularization parameter. 
Henceforth, let the solution of Eq. (\ref{equation_RKS_optimziation}) or (\ref{equation_RKS_optimziation_practice}) be $\{\alpha_i\}_{i=1}^p$ and the randomly sampled parameters be $\{w_i\}_{i=1}^p$.
We put into into Eq. (\ref{equation_RKS_representer_theorem}) but with only $p$ components of summation to result the solution of RKS, denoted by $\widehat{g}'$:
\begin{align}\label{equation_RKS_final_solution}
\widehat{g}'(\b{x}) = \sum_{i=1}^p \alpha_i\, \phi(\b{x}; w_i).
\end{align}
In the following, we show that Eq. (\ref{equation_RKS_final_solution}) minimizes the empirical risk very well even when the parameters $\{w_i\}_{i=1}^p$ are randomly selected. Hence, RKS is a nonlinear random projection which estimates the labels $\{l_i\}_{i=1}^n$ with a good approximation. 

Consider the set of continuous-version of functions in Eq. (\ref{equation_RKS_representer_theorem}) as:
\begin{align}
\mathcal{G} := \Big\{g(\b{x}) = \int_\Omega \alpha(w)\, \phi(\b{x}; w)\, dw \, \Big| \, |\alpha(w)| \leq c\, \mathbb{P}(w) \Big\}.
\end{align}
Let $g^*$ be a function in $\mathcal{G}$ and $\{w_i\}_{i=1}^p$ be sampled i.i.d. from $\mathbb{P}(w)$. 
The output of RKS, which is Eq. (\ref{equation_RKS_final_solution}), lies in the random set:
\begin{align}
\widehat{\mathcal{G}} := \Big\{\widehat{g} = \sum_{i=1}^p \alpha_i\, \phi(\b{x}; w_i)\, \Big| \, |\alpha_i| \leq \frac{c}{p}, \forall i \Big\}.
\end{align}


\begin{lemma}[Bound on the approximation error {\citep[Lemma 2]{rahimi2008weighted}}]\label{lemma_RKS_approximation_error_bound}
Suppose the loss $\ell(l,l')$ is $L$-Lipschitz. Let $g^* \in \mathcal{G}$ be the minimizer of the true risk over $\mathcal{G}$.
For $\delta >0$, with probability at least $(1-\delta)$, there exists $\widehat{g} \in \widehat{G}$ satisfying:
\begin{align}
R_t(\widehat{g}) - R_t(g^*) \leq \frac{Lc}{\sqrt{p}} \Bigg(1 + \sqrt{2 \ln(\frac{1}{\delta}})\Bigg).
\end{align}
The term $R_t(\widehat{g}) - R_t(g^*)$ is the approximation error and the above equation is an upper-bound on it. 
\end{lemma}

\begin{lemma}[Bound on the estimation error {\citep[Lemma 3]{rahimi2008weighted}}]\label{lemma_RKS_estimation_error_bound}
Suppose the loss can be stated as $\ell(l,l') = \ell(l l')$ and it is $L$-Lipschitz. 
For $\delta >0$ and for all $\widehat{g} \in \widehat{\mathcal{G}}$, with probability at least $(1-\delta)$, we have:
\begin{align}\label{equation_RKS_estimation_error_bound}
|R_t(\widehat{g}) - R_e(\widehat{g})| \leq \frac{1}{\sqrt{n}} \Bigg(4Lc + 2|\ell(0)| + Lc \sqrt{\frac{1}{2} \ln(\frac{1}{\delta}})\Bigg).
\end{align}
The term $|R_t(\widehat{g}) - R_e(\widehat{g})|$ is the estimation error and the above equation is an upper-bound on it. 
\end{lemma}

\begin{theorem}[Bound on error of RKS {\citep[Theorem 1]{rahimi2008weighted}}]\label{theorem_RKS_error}
Let the activation function $\phi(.)$ be bounded, i.e., $\sup_{\b{x}, w}|\phi(\b{x}; w)| \leq 1$. 
Suppose the loss can be stated as $\ell(l,l') = \ell(l l')$ and it is $L$-Lipschitz. 
The output of RKS, which is Eq. (\ref{equation_RKS_final_solution}), is $\widehat{g}' \in \widehat{\mathcal{G}}$ satisfying:
\begin{align}
R_t(\widehat{g}') - \min_{g \in \mathcal{G}} R_t(g) \leq \mathcal{O}\Bigg(\Big(\frac{1}{\sqrt{n}} + \frac{1}{\sqrt{p}}\Big) Lc \sqrt{\ln(\frac{1}{\delta})}\Bigg),
\end{align}
with probability at least $(1 - 2\delta)$.
\end{theorem}
\begin{proof}
Let $g^* \in \mathcal{G}$ be a minimizer of true risk over $\mathcal{G}$, $\widehat{g}' \in \widehat{\mathcal{G}}$ be a minimizer of empirical risk over $\widehat{\mathcal{G}}$, and $\widehat{g}^* \in \widehat{\mathcal{G}}$ be a minimizer of true risk over $\widehat{\mathcal{G}}$.
We have:
\begin{align}
R_t(\widehat{g}')& - R_t(g^*) \overset{(a)}{=} R_t(\widehat{g}') - R_t(\widehat{g}^*) + R_t(\widehat{g}^*) - R_t(g^*) \nonumber\\
&\leq | R_t(\widehat{g}') - R_t(\widehat{g}^*) | + R_t(\widehat{g}^*) - R_t(g^*). \label{equation_RKS_risk_diff_0}
\end{align}
By Lemma \ref{lemma_RKS_estimation_error_bound}, with probability at least $(1-\delta)$, we have:
\begin{align}
& | R_t(\widehat{g}^*) - R_e(\widehat{g}^*) | \leq \epsilon_\text{est}, \label{equation_RKS_risk_diff_1} \\
& | R_t(\widehat{g}') - R_e(\widehat{g}') | \leq \epsilon_\text{est}, \label{equation_RKS_risk_diff_2}
\end{align}
where $\epsilon_\text{est}$ is the right-hand side of Eq. (\ref{equation_RKS_estimation_error_bound}):
\begin{align}\label{equation_RKS_epsilon_est}
\epsilon_\text{est} := \frac{1}{\sqrt{n}} \Bigg(4Lc + 2|\ell(0)| + Lc \sqrt{\frac{1}{2} \ln(\frac{1}{\delta}})\Bigg).
\end{align}
We said that $\widehat{g}'$ is the minimizer of empirical risk over $\widehat{\mathcal{G}}$ so:
\begin{align}\label{equation_RKS_risk_diff_3}
R_e(\widehat{g}') \leq R_e(\widehat{g}^*).
\end{align}
From Eqs. (\ref{equation_RKS_risk_diff_1}), (\ref{equation_RKS_risk_diff_2}), and (\ref{equation_RKS_risk_diff_3}), with probability at least $(1-\delta)$, we have:
\begin{align}
| R_t(\widehat{g}')& - R_t(\widehat{g}^*) | \leq 2 \epsilon_\text{est} \nonumber\\
&\overset{(\ref{equation_RKS_epsilon_est})}{=} \frac{2}{\sqrt{n}} \Bigg(4Lc + 2|\ell(0)| + Lc \sqrt{\frac{1}{2} \ln(\frac{1}{\delta}})\Bigg). \label{equation_RKS_risk_diff_4}
\end{align}
On the other hand, the second term in Eq. (\ref{equation_RKS_risk_diff_0}) is the approximation error and by Lemma \ref{lemma_RKS_approximation_error_bound}, with probability at least $(1-\delta)$, we have:
\begin{align}\label{equation_RKS_risk_diff_5}
R_t(\widehat{g}^*) - R_t(g^*) \leq \epsilon_\text{app} := \frac{Lc}{\sqrt{p}} \Bigg(1 + \sqrt{2 \ln(\frac{1}{\delta}})\Bigg).
\end{align}
Using Eqs. (\ref{equation_RKS_risk_diff_0}), (\ref{equation_RKS_risk_diff_4}), and (\ref{equation_RKS_risk_diff_5}) and by Bonferroni's inequality or the so-called union bound \cite{bonferroni1936teoria}, i.e. Eq. (\ref{equation_Bonferroni_inequality}), we have:
\begin{align*}
R_t(\widehat{g}')& - R_t(g^*) \leq 2\epsilon_\text{est} + \epsilon_\text{app} \\
&= \frac{2}{\sqrt{n}} \Bigg(4Lc + 2|\ell(0)| + Lc \sqrt{\frac{1}{2} \ln(\frac{1}{\delta}})\Bigg) \\
&~~~~~ + \frac{Lc}{\sqrt{p}} \Bigg(1 + \sqrt{2 \ln(\frac{1}{\delta}})\Bigg) \\
&= \mathcal{O}\Bigg(\Big(\frac{1}{\sqrt{n}} + \frac{1}{\sqrt{p}}\Big) Lc \sqrt{\ln(\frac{1}{\delta})}\Bigg),
\end{align*}
with probability at least $(1-\delta-
\delta) = (1-2\delta)$. Q.E.D.
\end{proof}

\section{Other Methods for Nonlinear Random Projection}\label{section_other_nonlinear_random_projections}

In addition to RFF and RKS, there exist other, but similar, approaches for nonlinear random projection, such as ELM, random weights in neural network, and ensemble of random projections. In the following, we introduce these approaches. 

\subsection{Extreme Learning Machine}

Extreme Learning Machine (ELM) was initially proposed for regression as a feed-forward neural network with one hidden layer \cite{huang2006extreme}. 
It was then improved for a multi-layer perceptron network \cite{tang2015extreme}. Its extension to multi-class classification was done in \cite{huang2011extreme}. 
ELM is a feed-forward neural network whose all layers except the last layer are random. Let the sample size of training data be $n$, their dimensionality be $d$-dimensional, the one-to-last layer have $d'$ neurons, and the last layer has $p$ neurons where $p$ is the dimensionality of labels. Note that in classification task, labels are one-hot encoded so $p$ is the number of classes. Every layer except the last layer has a possibly nonlinear activation function and behaves like an RKS; although, it is learned by backpropagation and not Eq. (\ref{equation_RKS_optimziation_practice}). 
Let $\b{\beta} \in \mathbb{R}^{d' \times p}$ be the matrix of weights for the last layer. If the $n$ outputs of one-to-last layer be stacked in the matrix $\b{H} \in \mathbb{R}^{d' \times n}$ and their target desired labels are $\b{T} \in \mathbb{R}^{p \times n}$, we desire to have $\b{\beta}^\top \b{H} = \b{T}$.
After randomly sampling the weights of all layers except $\b{\beta}$, ELM learns the weights $\b{\beta}$ by solving a least squares problem:
\begin{equation}\label{equation_ELM_optimization}
\begin{aligned}
& \underset{\b{\beta} \in \mathbb{R}^{d' \times p}}{\text{minimize}}
& & \|\b{\beta}^\top \b{H} - \b{T}\|_F^2 + \lambda\|\b{\beta}\|_F^2,
\end{aligned}
\end{equation}
where $\|.\|_F$ is the Frobenius norm and $\lambda>0$ is the regularization parameter. In other words, the last layer of ELM behaves like a (Ridge) linear regression \cite{hoerl1970ridge}. 

It is shown in \cite{huang2006universal} that the random weights in neural network with nonlinear activation functions are universal approximators \cite{hornik1989multilayer}. Therefore, ELM works well enough for any classification and regression task. 
This shows the connection between ELM and RKS because both work on random weights on network but with slightly different approaches. This connection can be seen more if we interpret ELM using kernels \cite{huang2014insight}. 
There exists several surveys on ELM, such as \cite{huang2011extreme2,huang2015trends}.

\subsection{Randomly Weighted Neural Networks}

After proposal of RKS and ELM, it was shown that a feed-forward or convolutional neural network whose all layers are random also work very well \cite{jarrett2009best}. In other words, a stack of several RKS models works very well because a random network can be seen as a stack of several nonlinear random projections. As each of these nonlinear random layers has an upper-bound on their probability of error (see Theorem \ref{theorem_RKS_error}), the total probability of error in the network is also bounded. 
Note that it does not matter whether the network is feed-forward or convolutional because in both architectures, the output of activation functions are projection by the weights of layers. 
It is shown in \cite{saxe2011random} that the convolutional pooling architectures are frequency selective and translation invariant even if their weights are random; therefore, randomly weighed convolutional network work well enough. 
This also explains why random initialization of neural networks before backpropagation is a good and acceptable initialization. 
The randomly weighed neural network has been used for object recognition \cite{jarrett2009best,chung2016random} and face recognition \cite{cox2011beyond}. 
Some works have even put a step further and have made almost everything in network, including weights and hyper-parameters like architecture, learning rate, and number of neurons random \cite{pinto2009high}. 

\subsubsection{Distance Preservation by Deterministic Layers}

Consider a layer of neural network with $p$ neurons to $d$ neurons. Let $\b{U} \in \mathbb{R}^{d \times p}$ be the weight matrix of layer and $g(.)$ be the Rectified Linear Unit (ReLU) activation function \cite{glorot2011deep}. 
The layer can be modeled as a linear projection followed by ReLU activation function, i.e., $g(\b{U}^\top \b{x}_j)$.
The following lemma shows that for deterministic weights $\b{U}$, the distances are preserved. 
\begin{lemma}[Preservation of Euclidean distance by a layer \cite{bruna2013learning,bruna2014signal}]
Consider a network layer $\b{U} \in \mathbb{R}^{d \times p}$.
If the output of the previous layer for two points are $\b{x}_i, \b{x}_j \in \mathbb{R}^d$, we have:
\begin{align}
L_1 \|\b{x}_i - \b{x}_j\|_2 \leq \|g(\b{U}^\top \b{x}_i) -& g(\b{U}^\top \b{x}_j)\|_2 \nonumber\\
&\leq L_2 \|\b{x}_i - \b{x}_j\|_2, 
\end{align}
where $0 < L_1 \leq L_2$ are the Lipschitz constants. 
\end{lemma}
\begin{proof}
Proof is available in \cite{bruna2014signal}. 
\end{proof}

\subsubsection{Distance Preservation by Random Layers}

We showed that for deterministic weights, the distances are preserved. Here, we show that random weights also preserve the distances as well as the angles between points.  
Consider a network layer $\b{U} \in \mathbb{R}^{d \times p}$ whose elements are i.i.d. random values sampled from the Gaussian distribution. The activation function is ReLU and denoted by $g(.)$.
Suppose the output of the previous layer for two points are $\b{x}_i, \b{x}_j \in \mathbb{R}^d$. Assume the input data to the layer lie on a manifold $\mathcal{M}$ with the Gaussian mean width:
\begin{align*}
w_\mathcal{M} := \mathbb{E}[\sup_{\b{x}_i, \b{x}_j \in \mathcal{M}} \b{q}^\top (\b{x}_i - \b{x}_j)],
\end{align*}
where $\b{q} \in \mathbb{R}^d$ is a random vector whose elements are i.i.d. sampled from the Gaussian distribution. 
We assume that the manifold is normalized so the input data to the layer lie on a sphere. 
Let $\mathbb{B}_r^d \subset \mathbb{R}^d$ be the ball with radius $r$ in $d$-dimensional Euclidean space. Hence, we have $\mathcal{M} \subset \mathbb{B}_r^d$. 

\begin{theorem}[Preservation of Euclidean distance by a random layer {\citep[Theorem 3]{giryes2016deep}}]\label{theorem_random_layer_Euclidean_preservation}
Consider a random network layer $\b{U} \in \mathbb{R}^{d \times p}$ where the output of the previous layer for two points are $\b{x}_i, \b{x}_j \in \mathbb{R}^d$. Suppose $\mathcal{M} \subset \mathbb{B}_r^d$.
Assume the angle between $\b{x}_i$ and $\b{x}_j$ is denoted by:
\begin{align*}
\theta_{i,j} := \cos^{-1}\Big(\frac{\b{x}_i^\top \b{x}_j}{\|\b{x}_i\|_2 \|\b{x}_j\|_2}\Big),
\end{align*}
and satisfies $0 \leq \theta_{i,j} \leq \pi$. If $p \geq c\, \delta^{-4} w_\mathcal{M}$ with $c$ as a constant, with high probability, we have:
\begin{align}
\Big| &\|g(\b{U}^\top \b{x}_i) - g(\b{U}^\top \b{x}_j)\|_2^2 \nonumber\\
&- \Big( \frac{1}{2} \|\b{x}_i - \b{x}_j\|_2^2 + \|\b{x}_i\|_2 \|\b{x}_j\|_2 \psi(\b{x}_i, \b{x}_j) \Big) \Big| \leq \delta, \label{equation_random_layer_Euclidean_preservation}
\end{align}
where $\psi(\b{x}_i, \b{x}_j) \in [0,1]$ is defined as:
\begin{align}
\psi(\b{x}_i, \b{x}_j) := \frac{1}{\pi}\big(\!\sin(\theta_{i,j}) - \theta_{i,j} \cos(\theta_{i,j})\big).
\end{align}
\end{theorem}
\begin{proof}
Proof is available in {\citep[Appendix A]{giryes2016deep}}.
\end{proof}

\begin{theorem}[Preservation of angles by a random layer {\citep[Theorem 4]{giryes2016deep}}]\label{theorem_random_layer_angle_preservation}
Suppose the same assumptions as in Theorem \ref{theorem_random_layer_Euclidean_preservation} hold and let $\mathcal{M} \subset \mathbb{B}_1^d \setminus \mathbb{B}_\beta^d$ where $\delta \ll \beta^2 < 1$. 
Assume the angle between $g(\b{U}^\top \b{x}_i)$ and $g(\b{U}^\top \b{x}_j)$ is denoted by:
\begin{align*}
\theta'_{i,j} := \cos^{-1}\Big(\frac{(g(\b{U}^\top \b{x}_i))^\top g(\b{U}^\top \b{x}_j)}{\|g(\b{U}^\top \b{x}_i)\|_2 \|g(\b{U}^\top \b{x}_j)\|_2}\Big),
\end{align*}
and satisfies $0 \leq \theta'_{i,j} \leq \pi$. 
With high probability, we have:
\begin{align}\label{equation_random_layer_angle_preservation}
\Big| &\cos(\theta'_{i,j}) - \big(\!\cos(\theta_{i,j}) + \psi(\b{x}_i, \b{x}_j)\big) \Big| \leq \frac{15 \delta}{\beta^2 - 2 \delta}. 
\end{align}
\end{theorem}
\begin{proof}
Proof is available in {\citep[Appendix B]{giryes2016deep}}.
\end{proof}

\begin{corollary}[Preservation of distances and angles by a random layer {\citep[Corollary 5]{giryes2016deep}}]\label{corollary_random_layer_preservation_distances_corollary}
By a random layer with weights $\b{U}$ and ReLU activation function $g(.)$, for every two points $\b{x}_i$ and $\b{x}_j$ as inputs to the layer, we have with high probability that:
\begin{equation}\label{equation_random_layer_preservation_distances_corollary}
\begin{aligned}
\frac{1}{2} \|\b{x}_i - \b{x}_j\|_2^2 - \delta \leq\, &\|g(\b{U}^\top \b{x}_i) - g(\b{U}^\top \b{x}_j)\|_2^2 \\
&\leq \|\b{x}_i - \b{x}_j\|_2^2 + \delta.
\end{aligned}
\end{equation}
A similar expression can be stated for preserving the angle by the layer. 
\end{corollary}
\begin{proof}
The relation of $\psi(\b{x}_i, \b{x}_j)$ and $\cos(\b{x}_i, \b{x}_j)$ is depicted in Fig. \ref{figure_plot_randomNet} for $\theta_{i,j} \in [0, \pi]$. As this figure shows:
\begin{align*}
& \theta_{i,j}=0,\,\, \cos(\b{x}_i, \b{x}_j)=1 \implies \psi(\b{x}_i, \b{x}_j)=0, \\
& \theta_{i,j}=\pi,\,\, \cos(\b{x}_i, \b{x}_j)=0 \implies \psi(\b{x}_i, \b{x}_j)=1.
\end{align*}
Therefore, if two points are similar, i.e. their angle is zero, we have $\psi(\b{x}_i, \b{x}_j)=0$. Having $\psi(\b{x}_i, \b{x}_j)=0$ in Eqs. (\ref{equation_random_layer_Euclidean_preservation}) and (\ref{equation_random_layer_angle_preservation}) shows that when the two input points $\b{x}_i, \b{x}_j$ to a layer are similar, we almost have:
\begin{align*}
& \|g(\b{U}^\top \b{x}_i) - g(\b{U}^\top \b{x}_j)\|_2^2 \approx \|\b{x}_i - \b{x}_j\|_2^2, \\
& \cos(\theta'_{i,j}) \approx \cos(\theta_{i,j}). 
\end{align*}
This proves preserving both Euclidean distance and angle of points by a network layer with random weights and ReLU activation function. Q.E.D.
\end{proof}

\begin{figure}[!t]
\centering
\includegraphics[width=3.2in]{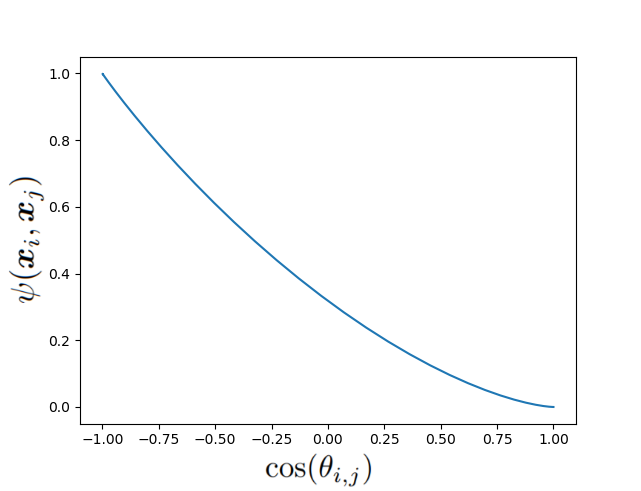}
\caption{Relation of $\psi(\b{x}_i, \b{x}_j)$ with $\cos(\theta_{i,j})$ for $\theta_{i,j} \in [0, \pi]$.}
\label{figure_plot_randomNet}
\end{figure}

It is noteworthy that a survey on randomness in neural networks exists \cite{scardapane2017randomness}.

\subsection{Ensemble of Random Projections}

There exist several ensemble methods for random projections where an ensemble of random projections is used. Ensemble methods benefit from model averaging \cite{hoeting1999bayesian} and bootstrap aggregating (bagging) \cite{breiman1996bagging}. Bagging reduces the estimation variance especially if the models are not correlated \cite{ghojogh2019theory}. As the random projection models are random, they are uncorrelated; hence, an ensemble of them can improve performance. 

The paper \cite{giryes2016deep}, introduced before, has stated that a similar analysis to the analysis of ReLU activation function can be done for max pooling \cite{scherer2010evaluation}.
Consider an ensemble of $m$ random projection matrices $\{\b{U}_j \in \mathbb{R}^{d \times p}\}_{j=1}^m$. We can transform data as a max pooling over this ensemble \cite{karimi2017ensembles}: 
\begin{align}
f(\b{x}) &:= \max\{f_1(\b{x}), \dots, f_m(\b{x})\} \nonumber\\
&= \max\{\b{U}_1^\top \b{x}, \dots, \b{U}_m^\top \b{x}\}.
\end{align}
This max pooling is performed element-wise, similar to element-wise operation of ReLU activation function introduced before. 

We can also have an ensemble of random projections for the classification task. One of the algorithms for this is \cite{schclar2009random} which we explain in the following. 
We sample projection matrices randomly where their elements are i.i.d. Also, we normalize the columns of projection matrices to have unit length.
We apply the ensemble of projection matrices $\{\b{U}_j \in \mathbb{R}^{d \times p}\}_{j=1}^m$ to the $n$ training data points $\b{X} \in \mathbb{R}^{d \times n}$ to have $\{f_1(\b{X}), \dots, f_m(\b{X})\} := \{\b{U}_1^\top \b{X}, \dots, \b{U}_m^\top \b{X}\}$. 
Let the labels of training data be denoted by $\{l_i\}_{i=1}^n$. 
We train an arbitrary classifier model using these projected data and their labels. In other words, the model is trained by $\{(\{\b{U}_1^\top \b{x}_i\}_{i=1}^n, \{l_i\}_{i=1}^n), \dots, (\{\b{U}_m^\top \b{x}_i\}_{i=1}^n, \{l_i\}_{i=1}^n)\}$. In the test phase, the test point is projected onto the column spaces of the $m$ projection matrices. These projections are fed to the classifier and the final prediction of label is found by majority voting. 
Note that there is another more sophisticated ensemble method for random projections \cite{cannings2015random} which divides dataset into disjoint sets. In each set an ensemble of random projections is performed. In each set, the validation error of random projections are calculated and the random projection with smallest validation error is selected. After training, majority voting is performed for predicting the test label. 

\section{Conclusion}\label{section_conclusion}

In this paper, we introduced the theory of linear and nonlinear random projections. We explained the JL lemma and its proof. Sparse random projection using $\ell_1$ norm, low-rank matrix approximation by random projection, and approximate nearest neighbor search by random projection onto hypercube were covered. Then, RFF, RKS, ELM, random neural networks, and ensemble of random projections were explained. 
For brevity, we did not some methods for random projection, such as bilateral random projection \cite{zhou2012bilateral} and dual random projection \cite{zhang2013recovering}. 

\section*{Acknowledgement}

The authors hugely thank Rachel Ward and Edo Liberty whose slides \cite{ward2014dimension,liberty2006random} partially covered some materials in this tutorial paper.  



\bibliography{References}

\begin{thebibliography}{84}
\providecommand{\natexlab}[1]{#1}
\providecommand{\url}[1]{\texttt{#1}}
\expandafter\ifx\csname urlstyle\endcsname\relax
  \providecommand{\doi}[1]{doi: #1}\else
  \providecommand{\doi}{doi: \begingroup \urlstyle{rm}\Url}\fi

\bibitem[Achlioptas(2001)]{achlioptas2001database}
Achlioptas, Dimitris.
\newblock Database-friendly random projections.
\newblock In \emph{Proceedings of the twentieth ACM SIGMOD-SIGACT-SIGART
  symposium on Principles of database systems}, pp.\  274--281, 2001.

\bibitem[Achlioptas(2003)]{achlioptas2003database}
Achlioptas, Dimitris.
\newblock Database-friendly random projections: {Johnson}-{Lindenstrauss} with
  binary coins.
\newblock \emph{Journal of computer and System Sciences}, 66\penalty0
  (4):\penalty0 671--687, 2003.

\bibitem[Ailon \& Chazelle(2006)Ailon and Chazelle]{ailon2006approximate}
Ailon, Nir and Chazelle, Bernard.
\newblock Approximate nearest neighbors and the fast {Johnson}-{Lindenstrauss}
  transform.
\newblock In \emph{Proceedings of the thirty-eighth annual ACM symposium on
  Theory of computing}, pp.\  557--563, 2006.

\bibitem[Aizerman(1964)]{aizerman1964theoretical}
Aizerman, Mark~A.
\newblock Theoretical foundations of the potential function method in pattern
  recognition learning.
\newblock \emph{Automation and remote control}, 25:\penalty0 821--837, 1964.

\bibitem[Andoni \& Indyk(2006)Andoni and Indyk]{andoni2006near}
Andoni, Alexandr and Indyk, Piotr.
\newblock Near-optimal hashing algorithms for approximate nearest neighbor in
  high dimensions.
\newblock In \emph{2006 47th annual IEEE symposium on foundations of computer
  science (FOCS'06)}, pp.\  459--468. IEEE, 2006.

\bibitem[Arriaga \& Vempala(1999)Arriaga and Vempala]{arriaga1999algorithmic}
Arriaga, Rosa~I and Vempala, Santosh.
\newblock Algorithmic theories of learning.
\newblock In \emph{Foundations of Computer Science}, volume~5, 1999.

\bibitem[Baraniuk et~al.(2006)Baraniuk, Davenport, DeVore, and
  Wakin]{baraniuk2006johnson}
Baraniuk, Richard, Davenport, Mark, DeVore, Ronald, and Wakin, Michael.
\newblock The {Johnson}-{Lindenstrauss} lemma meets compressed sensing.
\newblock Technical report, Rice University, 2006.

\bibitem[Bartal et~al.(2011)Bartal, Recht, and
  Schulman]{bartal2011dimensionality}
Bartal, Yair, Recht, Ben, and Schulman, Leonard~J.
\newblock Dimensionality reduction: beyond the {Johnson}-{Lindenstrauss} bound.
\newblock In \emph{Proceedings of the twenty-second annual ACM-SIAM symposium
  on Discrete Algorithms}, pp.\  868--887. SIAM, 2011.

\bibitem[Berinde et~al.(2008)Berinde, Gilbert, Indyk, Karloff, and
  Strauss]{berinde2008combining}
Berinde, Radu, Gilbert, Anna~C, Indyk, Piotr, Karloff, Howard, and Strauss,
  Martin~J.
\newblock Combining geometry and combinatorics: A unified approach to sparse
  signal recovery.
\newblock In \emph{2008 46th Annual Allerton Conference on Communication,
  Control, and Computing}, pp.\  798--805. IEEE, 2008.

\bibitem[Bonferroni(1936)]{bonferroni1936teoria}
Bonferroni, Carlo.
\newblock Teoria statistica delle classi e calcolo delle probabilita.
\newblock \emph{Pubblicazioni del R Istituto Superiore di Scienze Economiche e
  Commericiali di Firenze}, 8:\penalty0 3--62, 1936.

\bibitem[Breiman(1996)]{breiman1996bagging}
Breiman, Leo.
\newblock Bagging predictors.
\newblock \emph{Machine learning}, 24\penalty0 (2):\penalty0 123--140, 1996.

\bibitem[Brinkman \& Charikar(2005)Brinkman and
  Charikar]{brinkman2005impossibility}
Brinkman, Bo and Charikar, Moses.
\newblock On the impossibility of dimension reduction in $\ell_1$.
\newblock \emph{Journal of the ACM (JACM)}, 52\penalty0 (5):\penalty0 766--788,
  2005.

\bibitem[Bruna et~al.(2013)Bruna, Szlam, and LeCun]{bruna2013learning}
Bruna, Joan, Szlam, Arthur, and LeCun, Yann.
\newblock Learning stable group invariant representations with convolutional
  networks.
\newblock In \emph{ICLR Workshop}, 2013.

\bibitem[Bruna et~al.(2014)Bruna, Szlam, and LeCun]{bruna2014signal}
Bruna, Joan, Szlam, Arthur, and LeCun, Yann.
\newblock Signal recovery from pooling representations.
\newblock In \emph{International conference on machine learning}, pp.\
  307--315. PMLR, 2014.

\bibitem[Candes et~al.(2006)Candes, Romberg, and Tao]{candes2006stable}
Candes, Emmanuel~J, Romberg, Justin~K, and Tao, Terence.
\newblock Stable signal recovery from incomplete and inaccurate measurements.
\newblock \emph{Communications on Pure and Applied Mathematics: A Journal
  Issued by the Courant Institute of Mathematical Sciences}, 59\penalty0
  (8):\penalty0 1207--1223, 2006.

\bibitem[Cannings \& Samworth(2015)Cannings and Samworth]{cannings2015random}
Cannings, Timothy~I and Samworth, Richard~J.
\newblock Random-projection ensemble classification.
\newblock \emph{arXiv preprint arXiv:1504.04595}, 2015.

\bibitem[Charikar \& Sahai(2002)Charikar and Sahai]{charikar2002dimension}
Charikar, Moses and Sahai, Amit.
\newblock Dimension reduction in the $\ell_1$ norm.
\newblock In \emph{The 43rd Annual IEEE Symposium on Foundations of Computer
  Science, 2002. Proceedings.}, pp.\  551--560. IEEE, 2002.

\bibitem[Chitta et~al.(2012)Chitta, Jin, and Jain]{chitta2012efficient}
Chitta, Radha, Jin, Rong, and Jain, Anil~K.
\newblock Efficient kernel clustering using random {Fourier} features.
\newblock In \emph{2012 IEEE 12th International Conference on Data Mining},
  pp.\  161--170. IEEE, 2012.

\bibitem[Chung et~al.(2016)Chung, Shafiee, and Wong]{chung2016random}
Chung, Audrey~G, Shafiee, Mohammad~Javad, and Wong, Alexander.
\newblock Random feature maps via a layered random projection ({LARP})
  framework for object classification.
\newblock In \emph{2016 IEEE International Conference on Image Processing
  (ICIP)}, pp.\  246--250. IEEE, 2016.

\bibitem[Cox \& Pinto(2011)Cox and Pinto]{cox2011beyond}
Cox, David and Pinto, Nicolas.
\newblock Beyond simple features: A large-scale feature search approach to
  unconstrained face recognition.
\newblock In \emph{2011 IEEE International Conference on Automatic Face \&
  Gesture Recognition (FG)}, pp.\  8--15. IEEE, 2011.

\bibitem[Dasgupta \& Gupta(1999)Dasgupta and Gupta]{dasgupta1999elementary}
Dasgupta, Sanjoy and Gupta, Anupam.
\newblock An elementary proof of the {Johnson}-{Lindenstrauss} lemma.
\newblock \emph{International Computer Science Institute, Technical Report},
  22\penalty0 (1):\penalty0 1--5, 1999.

\bibitem[Dasgupta \& Gupta(2003)Dasgupta and Gupta]{dasgupta2003elementary}
Dasgupta, Sanjoy and Gupta, Anupam.
\newblock An elementary proof of a theorem of {Johnson} and {Lindenstrauss}.
\newblock \emph{Random Structures \& Algorithms}, 22\penalty0 (1):\penalty0
  60--65, 2003.

\bibitem[Donoho(2006)]{donoho2006compressed}
Donoho, David~L.
\newblock Compressed sensing.
\newblock \emph{IEEE Transactions on information theory}, 52\penalty0
  (4):\penalty0 1289--1306, 2006.

\bibitem[Eckart \& Young(1936)Eckart and Young]{eckart1936approximation}
Eckart, Carl and Young, Gale.
\newblock The approximation of one matrix by another of lower rank.
\newblock \emph{Psychometrika}, 1\penalty0 (3):\penalty0 211--218, 1936.

\bibitem[Frankl \& Maehara(1988)Frankl and Maehara]{frankl1988johnson}
Frankl, Peter and Maehara, Hiroshi.
\newblock The {Johnson}-{Lindenstrauss} lemma and the sphericity of some
  graphs.
\newblock \emph{Journal of Combinatorial Theory, Series B}, 44\penalty0
  (3):\penalty0 355--362, 1988.

\bibitem[Frieze et~al.(2004)Frieze, Kannan, and Vempala]{frieze2004fast}
Frieze, Alan, Kannan, Ravi, and Vempala, Santosh.
\newblock Fast {Monte}-{Carlo} algorithms for finding low-rank approximations.
\newblock \emph{Journal of the ACM (JACM)}, 51\penalty0 (6):\penalty0
  1025--1041, 2004.

\bibitem[Ghojogh \& Crowley(2019{\natexlab{a}})Ghojogh and
  Crowley]{ghojogh2019theory}
Ghojogh, Benyamin and Crowley, Mark.
\newblock The theory behind overfitting, cross validation, regularization,
  bagging, and boosting: tutorial.
\newblock \emph{arXiv preprint arXiv:1905.12787}, 2019{\natexlab{a}}.

\bibitem[Ghojogh \& Crowley(2019{\natexlab{b}})Ghojogh and
  Crowley]{ghojogh2019unsupervised}
Ghojogh, Benyamin and Crowley, Mark.
\newblock Unsupervised and supervised principal component analysis: Tutorial.
\newblock \emph{arXiv preprint arXiv:1906.03148}, 2019{\natexlab{b}}.

\bibitem[Ghojogh et~al.(2019)Ghojogh, Karray, and Crowley]{ghojogh2019fisher}
Ghojogh, Benyamin, Karray, Fakhri, and Crowley, Mark.
\newblock Fisher and kernel {Fisher} discriminant analysis: Tutorial.
\newblock \emph{arXiv preprint arXiv:1906.09436}, 2019.

\bibitem[Ghojogh et~al.(2021)Ghojogh, Ghodsi, Karray, and
  Crowley]{ghojogh2021reproducing}
Ghojogh, Benyamin, Ghodsi, Ali, Karray, Fakhri, and Crowley, Mark.
\newblock Reproducing kernel {Hilbert} space, {M}ercer's theorem,
  eigenfunctions, {N}ystr\"om method, and use of kernels in machine learning:
  Tutorial and survey.
\newblock \emph{arXiv preprint arXiv:2106.08443}, 2021.

\bibitem[Giryes et~al.(2016)Giryes, Sapiro, and Bronstein]{giryes2016deep}
Giryes, Raja, Sapiro, Guillermo, and Bronstein, Alex~M.
\newblock Deep neural networks with random {Gaussian} weights: A universal
  classification strategy?
\newblock \emph{IEEE Transactions on Signal Processing}, 64\penalty0
  (13):\penalty0 3444--3457, 2016.

\bibitem[Glorot et~al.(2011)Glorot, Bordes, and Bengio]{glorot2011deep}
Glorot, Xavier, Bordes, Antoine, and Bengio, Yoshua.
\newblock Deep sparse rectifier neural networks.
\newblock In \emph{Proceedings of the fourteenth international conference on
  artificial intelligence and statistics}, pp.\  315--323. JMLR Workshop and
  Conference Proceedings, 2011.

\bibitem[Hinneburg et~al.(2000)Hinneburg, Aggarwal, and
  Keim]{hinneburg2000nearest}
Hinneburg, Alexander, Aggarwal, Charu~C, and Keim, Daniel~A.
\newblock What is the nearest neighbor in high dimensional spaces?
\newblock In \emph{26th Internat. Conference on Very Large Databases}, pp.\
  506--515, 2000.

\bibitem[Hoeffding(1963)]{hoeffding1963probability}
Hoeffding, Wassily.
\newblock Probability inequalities for sums of bounded random variables.
\newblock \emph{Journal of the American Statistical Association}, 58\penalty0
  (301):\penalty0 13--30, 1963.

\bibitem[Hoerl \& Kennard(1970)Hoerl and Kennard]{hoerl1970ridge}
Hoerl, Arthur~E and Kennard, Robert~W.
\newblock Ridge regression: Biased estimation for nonorthogonal problems.
\newblock \emph{Technometrics}, 12\penalty0 (1):\penalty0 55--67, 1970.

\bibitem[Hoeting et~al.(1999)Hoeting, Madigan, Raftery, and
  Volinsky]{hoeting1999bayesian}
Hoeting, Jennifer~A, Madigan, David, Raftery, Adrian~E, and Volinsky, Chris~T.
\newblock Bayesian model averaging: a tutorial.
\newblock \emph{Statistical science}, 14\penalty0 (4):\penalty0 382--417, 1999.

\bibitem[Hornik et~al.(1989)Hornik, Stinchcombe, and
  White]{hornik1989multilayer}
Hornik, Kurt, Stinchcombe, Maxwell, and White, Halbert.
\newblock Multilayer feedforward networks are universal approximators.
\newblock \emph{Neural networks}, 2\penalty0 (5):\penalty0 359--366, 1989.

\bibitem[Huang et~al.(2015)Huang, Huang, Song, and You]{huang2015trends}
Huang, Gao, Huang, Guang-Bin, Song, Shiji, and You, Keyou.
\newblock Trends in extreme learning machines: A review.
\newblock \emph{Neural Networks}, 61:\penalty0 32--48, 2015.

\bibitem[Huang(2014)]{huang2014insight}
Huang, Guang-Bin.
\newblock An insight into extreme learning machines: random neurons, random
  features and kernels.
\newblock \emph{Cognitive Computation}, 6\penalty0 (3):\penalty0 376--390,
  2014.

\bibitem[Huang et~al.(2006{\natexlab{a}})Huang, Chen, Siew,
  et~al.]{huang2006universal}
Huang, Guang-Bin, Chen, Lei, Siew, Chee~Kheong, et~al.
\newblock Universal approximation using incremental constructive feedforward
  networks with random hidden nodes.
\newblock \emph{IEEE Trans. Neural Networks}, 17\penalty0 (4):\penalty0
  879--892, 2006{\natexlab{a}}.

\bibitem[Huang et~al.(2006{\natexlab{b}})Huang, Zhu, and
  Siew]{huang2006extreme}
Huang, Guang-Bin, Zhu, Qin-Yu, and Siew, Chee-Kheong.
\newblock Extreme learning machine: theory and applications.
\newblock \emph{Neurocomputing}, 70\penalty0 (1-3):\penalty0 489--501,
  2006{\natexlab{b}}.

\bibitem[Huang et~al.(2011{\natexlab{a}})Huang, Wang, and
  Lan]{huang2011extreme2}
Huang, Guang-Bin, Wang, Dian~Hui, and Lan, Yuan.
\newblock Extreme learning machines: a survey.
\newblock \emph{International journal of machine learning and cybernetics},
  2\penalty0 (2):\penalty0 107--122, 2011{\natexlab{a}}.

\bibitem[Huang et~al.(2011{\natexlab{b}})Huang, Zhou, Ding, and
  Zhang]{huang2011extreme}
Huang, Guang-Bin, Zhou, Hongming, Ding, Xiaojian, and Zhang, Rui.
\newblock Extreme learning machine for regression and multiclass
  classification.
\newblock \emph{IEEE Transactions on Systems, Man, and Cybernetics, Part B
  (Cybernetics)}, 42\penalty0 (2):\penalty0 513--529, 2011{\natexlab{b}}.

\bibitem[Indyk \& Motwani(1998)Indyk and Motwani]{indyk1998approximate}
Indyk, Piotr and Motwani, Rajeev.
\newblock Approximate nearest neighbors: towards removing the curse of
  dimensionality.
\newblock In \emph{Proceedings of the thirtieth annual ACM symposium on Theory
  of computing}, pp.\  604--613, 1998.

\bibitem[Jarrett et~al.(2009)Jarrett, Kavukcuoglu, Ranzato, and
  LeCun]{jarrett2009best}
Jarrett, Kevin, Kavukcuoglu, Koray, Ranzato, Marc'Aurelio, and LeCun, Yann.
\newblock What is the best multi-stage architecture for object recognition?
\newblock In \emph{2009 IEEE 12th international conference on computer vision},
  pp.\  2146--2153. IEEE, 2009.

\bibitem[Johnson \& Lindenstrauss(1984)Johnson and
  Lindenstrauss]{johnson1984extensions}
Johnson, William~B and Lindenstrauss, Joram.
\newblock Extensions of {Lipschitz} mappings into a {Hilbert} space.
\newblock \emph{Contemporary mathematics}, 26, 1984.

\bibitem[Karimi(2018)]{karimi2018exploring}
Karimi, Amir-Hossein.
\newblock Exploring new forms of random projections for prediction and
  dimensionality reduction in big-data regimes.
\newblock Master's thesis, University of Waterloo, 2018.

\bibitem[Karimi et~al.(2017)Karimi, Shafiee, Ghodsi, and
  Wong]{karimi2017ensembles}
Karimi, Amir~Hossein, Shafiee, Mohammad~Javad, Ghodsi, Ali, and Wong,
  Alexander.
\newblock Ensembles of random projections for nonlinear dimensionality
  reduction.
\newblock \emph{Journal of Computational Vision and Imaging Systems},
  3\penalty0 (1), 2017.

\bibitem[Krahmer \& Ward(2011)Krahmer and Ward]{krahmer2011new}
Krahmer, Felix and Ward, Rachel.
\newblock New and improved {Johnson}--{Lindenstrauss} embeddings via the
  restricted isometry property.
\newblock \emph{SIAM Journal on Mathematical Analysis}, 43\penalty0
  (3):\penalty0 1269--1281, 2011.

\bibitem[Krahmer \& Ward(2016)Krahmer and Ward]{krahmer2016unified}
Krahmer, Felix and Ward, Rachel.
\newblock A unified framework for linear dimensionality reduction in {L1}.
\newblock \emph{Results in Mathematics}, 70\penalty0 (1):\penalty0 209--231,
  2016.

\bibitem[Kushilevitz et~al.(1998)Kushilevitz, Ostrovsky, and
  Rabani]{kushilevitz1998algorithm}
Kushilevitz, Eyal, Ostrovsky, Rafail, and Rabani, Yuval.
\newblock An algorithm for approximate closest-point queries.
\newblock In \emph{Proc. 30th ACM Symp. on Theory of Computing}, pp.\
  614--623, 1998.

\bibitem[Kushilevitz et~al.(2000)Kushilevitz, Ostrovsky, and
  Rabani]{kushilevitz2000efficient}
Kushilevitz, Eyal, Ostrovsky, Rafail, and Rabani, Yuval.
\newblock Efficient search for approximate nearest neighbor in high dimensional
  spaces.
\newblock \emph{SIAM Journal on Computing}, 30\penalty0 (2):\penalty0 457--474,
  2000.

\bibitem[Larsen \& Nelson(2017)Larsen and Nelson]{larsen2017optimality}
Larsen, Kasper~Green and Nelson, Jelani.
\newblock Optimality of the {Johnson}-{Lindenstrauss} lemma.
\newblock In \emph{2017 IEEE 58th Annual Symposium on Foundations of Computer
  Science (FOCS)}, pp.\  633--638. IEEE, 2017.

\bibitem[Lee et~al.(2005)Lee, Mendel, and Naor]{lee2005metric}
Lee, James~R, Mendel, Manor, and Naor, Assaf.
\newblock Metric structures in {L}1: dimension, snowflakes, and average
  distortion.
\newblock \emph{European Journal of Combinatorics}, 26\penalty0 (8):\penalty0
  1180--1190, 2005.

\bibitem[Li et~al.(2006)Li, Hastie, and Church]{li2006very}
Li, Ping, Hastie, Trevor~J, and Church, Kenneth~W.
\newblock Very sparse random projections.
\newblock In \emph{Proceedings of the 12th ACM SIGKDD international conference
  on Knowledge discovery and data mining}, pp.\  287--296, 2006.

\bibitem[Li et~al.(2007)Li, Hastie, and Church]{li2007nonlinear}
Li, Ping, Hastie, Trevor~J, and Church, Kenneth~W.
\newblock Nonlinear estimators and tail bounds for dimension reduction in
  $\ell_1$ using {Cauchy} random projections.
\newblock \emph{Journal of Machine Learning Research}, 8\penalty0
  (Oct):\penalty0 2497--2532, 2007.

\bibitem[Li et~al.(2019)Li, Ton, Oglic, and Sejdinovic]{li2019towards}
Li, Zhu, Ton, Jean-Francois, Oglic, Dino, and Sejdinovic, Dino.
\newblock Towards a unified analysis of random {Fourier} features.
\newblock In \emph{International Conference on Machine Learning}, pp.\
  3905--3914. PMLR, 2019.

\bibitem[Liberty(2006)]{liberty2006random}
Liberty, Edo.
\newblock The random projection method; chosen chapters from {DIMACS} vol. 65
  by {Santosh} s. {Vempala}.
\newblock Technical report, Yale University, 2006.

\bibitem[Martinsson \& Rokhlin(2007)Martinsson and Rokhlin]{martinsson2007fast}
Martinsson, Per-Gunnar and Rokhlin, Vladimir.
\newblock A fast direct solver for scattering problems involving elongated
  structures.
\newblock \emph{Journal of Computational Physics}, 221\penalty0 (1):\penalty0
  288--302, 2007.

\bibitem[Matou{\v{s}}ek(2008)]{matouvsek2008variants}
Matou{\v{s}}ek, Ji{\v{r}}{\'\i}.
\newblock On variants of the {Johnson}--{Lindenstrauss} lemma.
\newblock \emph{Random Structures \& Algorithms}, 33\penalty0 (2):\penalty0
  142--156, 2008.

\bibitem[Meiser(1993)]{meiser1993point}
Meiser, Stefan.
\newblock Point location in arrangements of hyperplanes.
\newblock \emph{Information and Computation}, 106\penalty0 (2):\penalty0
  286--303, 1993.

\bibitem[Papadimitriou et~al.(2000)Papadimitriou, Raghavan, Tamaki, and
  Vempala]{papadimitriou2000latent}
Papadimitriou, Christos~H, Raghavan, Prabhakar, Tamaki, Hisao, and Vempala,
  Santosh.
\newblock Latent semantic indexing: A probabilistic analysis.
\newblock \emph{Journal of Computer and System Sciences}, 61\penalty0
  (2):\penalty0 217--235, 2000.

\bibitem[Pinto et~al.(2009)Pinto, Doukhan, DiCarlo, and Cox]{pinto2009high}
Pinto, Nicolas, Doukhan, David, DiCarlo, James~J, and Cox, David~D.
\newblock A high-throughput screening approach to discovering good forms of
  biologically inspired visual representation.
\newblock \emph{PLoS computational biology}, 5\penalty0 (11):\penalty0
  e1000579, 2009.

\bibitem[Plan \& Vershynin(2013)Plan and Vershynin]{plan2013one}
Plan, Yaniv and Vershynin, Roman.
\newblock One-bit compressed sensing by linear programming.
\newblock \emph{Communications on Pure and Applied Mathematics}, 66\penalty0
  (8):\penalty0 1275--1297, 2013.

\bibitem[Rahimi \& Recht(2007)Rahimi and Recht]{rahimi2007random}
Rahimi, Ali and Recht, Benjamin.
\newblock Random features for large-scale kernel machines.
\newblock In \emph{Advances in neural information processing systems},
  volume~20, 2007.

\bibitem[Rahimi \& Recht(2008{\natexlab{a}})Rahimi and
  Recht]{rahimi2008uniform}
Rahimi, Ali and Recht, Benjamin.
\newblock Uniform approximation of functions with random bases.
\newblock In \emph{2008 46th Annual Allerton Conference on Communication,
  Control, and Computing}, pp.\  555--561. IEEE, 2008{\natexlab{a}}.

\bibitem[Rahimi \& Recht(2008{\natexlab{b}})Rahimi and
  Recht]{rahimi2008weighted}
Rahimi, Ali and Recht, Benjamin.
\newblock Weighted sums of random kitchen sinks: replacing minimization with
  randomization in learning.
\newblock In \emph{Advances in neural information processing systems}, pp.\
  1313--1320, 2008{\natexlab{b}}.

\bibitem[Ramirez et~al.(2012)Ramirez, Arce, Otero, Paredes, and
  Sadler]{ramirez2012reconstruction}
Ramirez, Ana~B, Arce, Gonzalo~R, Otero, Daniel, Paredes, Jose-Luis, and Sadler,
  Brian~M.
\newblock Reconstruction of sparse signals from $\ell_1$ dimensionality-reduced
  {Cauchy} random projections.
\newblock \emph{IEEE Transactions on Signal Processing}, 60\penalty0
  (11):\penalty0 5725--5737, 2012.

\bibitem[Saxe et~al.(2011)Saxe, Koh, Chen, Bhand, Suresh, and
  Ng]{saxe2011random}
Saxe, Andrew~M, Koh, Pang~Wei, Chen, Zhenghao, Bhand, Maneesh, Suresh, Bipin,
  and Ng, Andrew~Y.
\newblock On random weights and unsupervised feature learning.
\newblock In \emph{International Conference on Machine Learning}, 2011.

\bibitem[Scardapane \& Wang(2017)Scardapane and Wang]{scardapane2017randomness}
Scardapane, Simone and Wang, Dianhui.
\newblock Randomness in neural networks: an overview.
\newblock \emph{Wiley Interdisciplinary Reviews: Data Mining and Knowledge
  Discovery}, 7\penalty0 (2):\penalty0 e1200, 2017.

\bibitem[Schclar \& Rokach(2009)Schclar and Rokach]{schclar2009random}
Schclar, Alon and Rokach, Lior.
\newblock Random projection ensemble classifiers.
\newblock In \emph{International Conference on Enterprise Information Systems},
  pp.\  309--316. Springer, 2009.

\bibitem[Scherer et~al.(2010)Scherer, M{\"u}ller, and
  Behnke]{scherer2010evaluation}
Scherer, Dominik, M{\"u}ller, Andreas, and Behnke, Sven.
\newblock Evaluation of pooling operations in convolutional architectures for
  object recognition.
\newblock In \emph{International conference on artificial neural networks},
  pp.\  92--101. Springer, 2010.

\bibitem[Shalev-Shwartz \& Ben-David(2014)Shalev-Shwartz and
  Ben-David]{shalev2014understanding}
Shalev-Shwartz, Shai and Ben-David, Shai.
\newblock \emph{Understanding machine learning: From theory to algorithms}.
\newblock Cambridge university press, 2014.

\bibitem[Slaney \& Casey(2008)Slaney and Casey]{slaney2008locality}
Slaney, Malcolm and Casey, Michael.
\newblock Locality-sensitive hashing for finding nearest neighbors [lecture
  notes].
\newblock \emph{IEEE Signal processing magazine}, 25\penalty0 (2):\penalty0
  128--131, 2008.

\bibitem[Sriperumbudur \& Szab{\'o}(2015)Sriperumbudur and
  Szab{\'o}]{sriperumbudur2015optimal}
Sriperumbudur, Bharath~K and Szab{\'o}, Zolt{\'a}n.
\newblock Optimal rates for random fourier features.
\newblock \emph{arXiv preprint arXiv:1506.02155}, 2015.

\bibitem[Sutherland \& Schneider(2015)Sutherland and
  Schneider]{sutherland2015error}
Sutherland, Danica~J and Schneider, Jeff.
\newblock On the error of random {Fourier} features.
\newblock \emph{arXiv preprint arXiv:1506.02785}, 2015.

\bibitem[Talagrand(1996)]{talagrand1996new}
Talagrand, Michel.
\newblock A new look at independence.
\newblock \emph{The Annals of probability}, pp.\  1--34, 1996.

\bibitem[Tang et~al.(2015)Tang, Deng, and Huang]{tang2015extreme}
Tang, Jiexiong, Deng, Chenwei, and Huang, Guang-Bin.
\newblock Extreme learning machine for multilayer perceptron.
\newblock \emph{IEEE transactions on neural networks and learning systems},
  27\penalty0 (4):\penalty0 809--821, 2015.

\bibitem[Vempala(2005)]{vempala2005random}
Vempala, Santosh~S.
\newblock \emph{The Random Projection Method}, volume 65 (Series in Discrete
  Mathematics and Theoretical Computer Science).
\newblock American Mathematical Society, 2005.

\bibitem[Ward(2014)]{ward2014dimension}
Ward, Rachel.
\newblock Dimension reduction via random projections.
\newblock Technical report, University of Texas at Austin, 2014.

\bibitem[Xie et~al.(2017)Xie, Li, and Xue]{xie2017survey}
Xie, Haozhe, Li, Jie, and Xue, Hanqing.
\newblock A survey of dimensionality reduction techniques based on random
  projection.
\newblock \emph{arXiv preprint arXiv:1706.04371}, 2017.

\bibitem[Yang et~al.(2012)Yang, Li, Mahdavi, Jin, and Zhou]{yang2012nystrom}
Yang, Tianbao, Li, Yu-Feng, Mahdavi, Mehrdad, Jin, Rong, and Zhou, Zhi-Hua.
\newblock Nystr{\"o}m method vs random {Fourier} features: A theoretical and
  empirical comparison.
\newblock In \emph{Advances in neural information processing systems},
  volume~25, pp.\  476--484, 2012.

\bibitem[Zhang et~al.(2013)Zhang, Mahdavi, Jin, Yang, and
  Zhu]{zhang2013recovering}
Zhang, Lijun, Mahdavi, Mehrdad, Jin, Rong, Yang, Tianbao, and Zhu, Shenghuo.
\newblock Recovering the optimal solution by dual random projection.
\newblock In \emph{Conference on Learning Theory}, pp.\  135--157. PMLR, 2013.

\bibitem[Zhou \& Tao(2012)Zhou and Tao]{zhou2012bilateral}
Zhou, Tianyi and Tao, Dacheng.
\newblock Bilateral random projections.
\newblock In \emph{2012 IEEE International Symposium on Information Theory
  Proceedings}, pp.\  1286--1290. IEEE, 2012.

\end{thebibliography}
\bibliographystyle{icml2016}

\end{document}